\documentclass{article}
\pdfoutput=1
\PassOptionsToPackage{numbers, compress}{natbib}

\usepackage[final]{neurips_2021}

\usepackage{xspace}
\usepackage{tikz}
\usepackage[colorlinks,
            linkcolor=red,
            anchorcolor=blue,
            citecolor=blue
            ]{hyperref}

\usepackage[utf8]{inputenc} 
\usepackage[T1]{fontenc}    
\usepackage{url}            
\usepackage{booktabs}       
\usepackage{amsfonts}       
\usepackage{nicefrac}       
\usepackage{microtype}      
\usepackage{xcolor}         
\usepackage{wrapfig}
\usepackage[compact]{titlesec}
\usepackage{sidecap}
\usepackage{smile}

\usepackage{enumitem}
\title{Reward-Free Model-Based Reinforcement Learning with Linear Function Approximation}

\author{%
  Weitong Zhang \\
  Department of Computer Science \\
  University of California, Los Angeles \\
  Los Angeles, CA 90095 \\
  \texttt{weightzero@cs.ucla.edu}
  \And 
  Dongruo Zhou \\
  Department of Computer Science \\
  University of California, Los Angeles \\
  Los Angeles, CA 90095 \\
  \texttt{drzhou@cs.ucla.edu}
  \And
  Quanquan Gu \\
  Department of Computer Science \\
  University of California, Los Angeles \\
  Los Angeles, CA 90095 \\
  \texttt{qgu@cs.ucla.edu}
}
\newcommand{\algname}{UCRL-RFE\xspace}
\newcommand{\newalgname}{UCRL-RFE+\xspace}
\newcolumntype{P}[1]{>{\centering\arraybackslash}m{#1}}
\usetikzlibrary{arrows,shapes,snakes,automata,backgrounds,petri}
\begin{document}

\maketitle
\begin{abstract}
We study the model-based reward-free reinforcement learning with linear function approximation for episodic Markov decision processes (MDPs). In this setting, the agent works in two phases. In the exploration phase, the agent interacts with the environment and collects samples without the reward. In the planning phase, the agent is given a specific reward function and uses samples collected from the exploration phase to learn a good policy. We propose a new provably efficient algorithm, called UCRL-RFE under the Linear Mixture MDP assumption, where the transition probability kernel of the MDP can be parameterized by a linear function over certain feature mappings defined on the triplet of state, action, and next state. We show that to obtain an $\epsilon$-optimal policy for arbitrary reward function, UCRL-RFE needs to sample at most $\tilde \cO(H^5d^2\epsilon^{-2})$ episodes during the exploration phase. Here, $H$ is the length of the episode, $d$ is the dimension of the feature mapping. We also propose a variant of UCRL-RFE using Bernstein-type bonus and show that it needs to sample at most $\tilde \cO(H^4d(H + d)\epsilon^{-2})$ to achieve an $\epsilon$-optimal policy. By constructing a special class of linear Mixture MDPs, we also prove that for any reward-free algorithm, it needs to sample at least $\tilde \Omega(H^2d\epsilon^{-2})$ episodes to obtain an $\epsilon$-optimal policy. 
Our upper bound matches the lower bound in terms of the dependence on $\epsilon$ and the dependence on $d$ if $H \ge d$. 

\end{abstract}

\section{Introduction}

In reinforcement learning (RL), the agent sequentially interacts with the environment and receives reward from it. In many real-world RL problems, the reward function is manually designed to encourage the desired behavior of the agent. Thus, engineers have to change the reward function time by time and train the agent to check whether it has achieved the desired behavior. In this case, RL algorithms need to be repeatedly executed with different reward functions and are therefore sample inefficient or even intractable. To tackle this challenge, \citet{jin2020reward} proposed a new reinforcement learning paradigm called \emph{Reward-Free Exploration} (RFE), which explores the environment without using any reward function. In detail, the reward-free RL algorithm consists of two phases. The first phase is called \emph{Exploration Phase}, where the algorithm explores the environment without receiving reward signals. The second phase is called \emph{Planning Phase}, where the algorithm is given a specific reward function and use the collected data in the first phase to learn the policy. They have shown that this exploration paradigm can learn a near-optimal policy in the planning phase given \emph{any} reward function after collecting polynomial number of episodes in the exploration phase. Follow up work \citep{kaufmann2020adaptive, menard2020fast, zhang2020nearly} proposed improved algorithms to achieve better or nearly optimal sample complexity.

All the aforementioned works are focused on the tabular Markov decision process (MDP), where the number of states and actions are finite. In practice, the number of states and actions can be large or even infinite, and therefore \emph{function approximation} is required for the sake of computational tractability and generalization. 
However, the understanding of function approximation for reward-free exploration, even under the simplest linear function approximation, remains underexplored, with only two notable related works \citep{wang2020reward,zanette2020provably}. Specifically, \citet{wang2020reward} studied \emph{linear MDPs} \citep{yang2019sample,jin2020provably}, where both the transition probability and the reward function admit linear representations, and proposed a reward-free RL algorithm with $\tilde \cO(d^3H^6\epsilon^{-2})$ sample complexity, where $d$ is the dimension of the linear representation, $H$ is the planning horizon, and $\epsilon$ is the required accuracy. They also proved that if the optimal state-action function is linear, then the reward-free exploration needs an exponential number of episodes in the planning horizon $H$ to learn a $\epsilon$-optimal policy. \citet{zanette2020provably} considered a slightly larger class of MDPs with \emph{low inherent Bellman error} \citep{zanette2020learning}, and proposed an algorithm with $\tilde \cO(d^3H^5\epsilon^{-2})$ sample complexity. However, both works assume the reward function is a linear function over some feature mapping. Moreover, the lower bound proved in \citep{wang2020reward} is for a very large class of MDPs where the optimal state-action function is linear, thus it is too conservative and cannot tell the information-theoretic limits of reward-free exploration for linear MDPs or related models.

In this paper, we seek a better understanding of the statistical efficiency for reward-free RL with linear function approximation. We propose two reward-free model-based RL algorithms for the finite-horizon episodic \emph{linear mixture/kernel MDP} \citep{modi2020sample,jia2020model,ayoub2020model,zhou2020provably}, where the transition probability kernel is a linear mixture model.
In detail, our contributions are highlighted as follows:
\vspace{-1ex}
\begin{itemize}[leftmargin = *]
\item We propose a new exploration-driven reward function and its corresponding pseudo value function for linear mixture MDPs, which will encourage the algorithm to explore the state-action pair with more uncertainty on the transition probability.
\item We propose a~\algname algorithm which guides the agent to explore the state space using the exploration-driven reward function and pseudo value functions. We prove an $\tilde \cO(H^5d^2\epsilon^{-2})$ sample complexity for \algname to achieve an $\epsilon$-optimal policy for any reward function for time-homogeneous MDP.
\item We further propose a \newalgname algorithm which uses a Bernstein-type exploration bonus. \newalgname can reduce the error caused by the exploration-driven reward function during the exploration phase. With a novel analysis based on total variance, we prove an $\tilde \cO(H^4d(H + d)\epsilon^{-2})$ sample complexity for \newalgname, which improves that of \algname by a factor of $\min\{H, d\}$.
\item By constructing a special class of linear mixture MDPs, we show that any reward-free algorithm needs to sample at least $\tilde \Omega(H^2d\epsilon^{-2})$ episodes to achieve an $\epsilon$-optimal policy for any reward function. This lower bound matches the upper bound of \newalgname in terms of the dependence on the accuracy $\epsilon$ and feature dimension $d$ when $H \ge d$. 
\end{itemize}
\vspace{-1ex}
\paragraph{Notation.} Scalars and constants are denoted by lower and upper case letters, respectively. Vectors are denoted by lower case bold face letters $\xb$, and matrices by upper case bold face letters $\Ab$. We denote by $[k]$ the set $\{1, 2, \cdots, k\}$ for positive integers $k$. For two non-negative sequence $\{a_n\}, \{b_n\}$, $a_n = \cO(b_n)$ means that there exists a positive constant $C$ such that $a_n \le Cb_n$, and we use $\tilde \cO(\cdot)$ to hide the $\log$ factor in $\cO(\cdot)$; $a_n = \Omega(b_n)$ means that there exists a positive constant $C$ such that $a_n \ge Cb_n$, and we use $\tilde \Omega(\cdot)$ to hide the $\log$ factor. $a_n = o(b_n)$ means that $\lim_{n \rightarrow \infty} a_n / b_n = 0$. We denote by $S, A$ as the cardinality of the state set $\cS$ and action set $\cA$ separately. For a vector $\xb  \in \RR^d$ and corresponding matrix $\Ab \in \RR^{d \times d}$, we define $\|\xb\|_{\Ab}^2=\xb^\top \Ab\xb$. We denote $[x]_{(0, H)} := \max\{\min\{x, H\}, 0\}$. For vector $\xb \in \RR^d$, we denote by $[\xb]_i$ the $i$-th element of $\xb$.
\section{Related Work}\label{sec:related}

\paragraph{Reinforcement Learning with Function Approximation.} Function approximation is extremely useful for RL when the state space and/or the action space are large or even infinite. To develop provable RL algorithms with linear function approximation, linear MDPs~\citep{jin2020provably} is probably the most widely assumed MDP model, where both the transition kernel and the reward function are linear functions of a given feature mapping. A line of works has developed RL algorithms with polynomial sample complexity or regret bounds under this setting, such as LSVI-UCB~\citep{wang2019optimism} and randomized LSVI \citep{zanette2020frequentist}. 
Besides the linear MDP, linear mixture/kernel MDPs~\citep{modi2020sample,jia2020model,ayoub2020model,zhou2020provably} has emerged as a new model which enables efficient RL with linear function approximation. In this setting, the transition kernel is a linear function over a feature mapping on the triplet of state, action, and next-state. Under this assumption, nearly minimax optimal regrets can be attained for both finite-horizon episodic MDPs and infinite-horizon discounted MDPs \citep{zhou2020nearly}. Note also that linear mixture MDPs do not require the reward function to be linear and therefore enables RL with arbitrary reward functions. Therefore, we also consider linear mixture MDPs in this paper.

\begin{table}[ht]
    \centering
    \begin{tabular}{P{9em}P{8em}P{3em}P{8em}P{5em}}
    \toprule
     Algorithm & Sample Complexity & Time Homo. & MDP Type & Model Based\\
     \midrule
     \citet{jin2020reward} & $\tilde \cO(H^5S^2A\epsilon^{-2})$ & $\times$ & Tabular & $\surd$\\
     RF-UCRL \citep{kaufmann2020adaptive} & $\tilde \cO(H^4S^2A\epsilon^{-2})$ & $\times$ & Tabular & $\surd$\\
     RF-Express~\citep{menard2020fast} & $\tilde \cO(H^3S^2A\epsilon^{-2})$  & $\times$ & Tabular & $\surd$\\
     SSTP~\citep{zhang2020nearly} & $\tilde \cO(H^2S^2A\epsilon^{-2})$ & $\surd$ & Tabular & $\surd$\\
     Lower bound \citep{jin2020reward} & $\Omega(H^2S^2A\epsilon^{-2})$ & $\surd$ & Tabular & $\surd$\\
     \midrule
     \citet{wang2020reward} & $\tilde \cO(H^6d^3\epsilon^{-2})$ & $\times$ & Linear MDP & $\times$\\
     FRANCIS~\citep{zanette2020learning} & $\tilde \cO(H^5d^3\epsilon^{-2})$ & $\surd$ & Linear MDP & $\times$\\
     \midrule
     UCRL-RFE (Alg.~\ref{alg:explore}) & $\tilde \cO(H^5d^2\epsilon^{-2})$ & $\surd$ & Linear Mixture & $\surd$\\
     UCRL-RFE+ (Alg.~\ref{alg:explore-bern}) & $\tilde \cO(H^4d(H + d)\epsilon^{-2})$ & $\surd$ & Linear Mixture & $\surd$\\
     Lower bound (Thm.~\ref{thm:mainL}) & $\tilde \Omega(H^2d\epsilon^{-2})$ & $\surd$ & Linear MDP/Linear Mixture & $\surd$\\
     \bottomrule
    \end{tabular}
    \caption{Comparison of episodic reward-free RL algorithms. 
    \emph{Time Homo.} stands for the MDP is a time-homogeneous, where the transition probabilities are the same at different stages of the episode. 
    \emph{Model Based} stands for the algorithm is a model-based algorithm ($\surd$) or a model-free algorithm ($\times$).}.
    \label{tab:sum}
    \vspace{-4ex}
\end{table}
\paragraph{Reward-Free Exploration.} As the first work on reward-free exploration,~\citet{jin2020reward} assigned each state an exploration-driven reward function at each round to guide the algorithm to do exploration. Then they utilized the EULER~\citep{zanette2019tighter} algorithm to minimize the total regret. Their algorithm achieves an $\tilde \cO(S^2AH^5\epsilon^{-2})$ sample complexity in the tabular setting to achieve an $\epsilon$-optimal policy, where $S$ is the number of states and $A$ is the number of actions. They also proved a sample complexity lower bound as $\tilde \Omega(S^2AH^2\epsilon^{-2})$.  \citet{kaufmann2020adaptive} extended the UCRL~\citep{auer2009near} algorithm to the reward-free exploration. Their algorithm RF-UCRL achieves a sample complexity of $\tilde \cO(S^2AH^4\epsilon^{-2})$, which improves that of~\cite{jin2020reward} by a factor of $H$. ~\citet{menard2020fast} proposed RF-Express algorithm by modifying the UCB-bonus of UCRL to making it decay faster and achieved a sample complexity of $\tilde \cO(S^2AH^3\epsilon^{-2})$. \citet{zhang2020nearly} proposed SSTP algorithm in the time-homogeneous setting, which achieves $\tilde \cO(S^2AH^2\epsilon^{-2})$ sample complexity, and matches the minimax lower bound provided in~\cite{jin2020reward} up to logarithmic factors. \citet{liu2020sharp} has shown the similarity between the self-play setting and reward-free setting. All of these works are for tabular MDPs.



Here we summarize and compare the related works on Reward Free Exploration in Table~\ref{tab:sum}. Notice that our lower bound $\Omega(H^2d\epsilon^{-2})$ for linear mixture MDPs can imply the same lower bound for linear MDPs and MDPs with low inherent Bellman error, using a similar argument used in~\cite{zhou2020nearly}. 

\section{Preliminaries}
We consider episodic Markov Decision Processes (MDP), which is denoted by a tuple $M(\cS, \cA, H, \{r_h\}_{h=1}^H, \PP)$. Here $\cS$ is the countable state space (may be infinite), $\cA$ is the action space, $H$ is the length of the episode, $r_h: \cS \times \cA \rightarrow [0,1]$ is the reward function. Without loss of generality, we assume the reward function $r_h$ is \emph{deterministic}. $\PP(s'|s,a)$ is the transition probability function which denotes the probability for state $s$ to transit to state $s'$ given action $a$ at step $h$. A policy $\pi_h: \cS \rightarrow \cA$ is a function which maps a state $s$ to an action $a$. 
We define the action-value function (i.e., Q-function) $Q_h^{\pi}(s,a)$ as follows:
\begin{align*}
    Q_h^{\pi}(s,a; \{r_h\}_h) = \EE\bigg[\sum_{h' = h}^H r_{h'}(s_{h'}, a_{h'})\bigg|s_h = s, a_h = a  \bigg], V_h^{\pi}(s; \{r_h\}_h) = Q_h^{\pi}(s,\pi_h(s); \{r_h\}_h).
\end{align*}
    
For simplicity, we denote $Q_h^{\pi}(s,a; r) = Q_h^{\pi}(s,a; \{r_h\}_h)$ and $V_h^{\pi}(s; r) = V_h^{\pi}(s; \{r_h\}_h)$. 
We define the optimal value function $\{V^*_h\}_{h=1}^H$ and the optimal action-value function $\{Q^*_h\}_{h=1}^H$ as $V^*_h(s; r) = \sup_{\pi}V_h^{\pi}(s; r)$ and $Q_h^*(s,a; r) = \sup_{\pi}Q_h^{\pi}(s,a; r)$ respectively.
For any function $V: \cS \rightarrow \RR$, we denote $[\PP V](s,a; r)=\EE_{s' \sim \PP(\cdot|s,a)}V(s'; r)$, and denote the variance of $V$ as
\begin{align}
    [\VV f](s, a) = [\PP f^2](s, a) - \big([\PP f](s, a)^2\big).\label{eq:realv}
\end{align}
In particular, we have the following Bellman equation, as well as the Bellman optimality equation:
\begin{align}
    Q_h^{\pi}(s,a; r) = r_h(s,a) + [\PP V_{h+1}^\pi](s,a; r), Q_h^*(s,a; r) = r_h(s,a) + [\PP V_{h+1}^*](s,a; r).\notag
\end{align}
In this paper, we focus on \emph{model-based} algorithms and consider the following \emph{linear mixture/kernel MDP} \citep{modi2020sample,jia2020model,ayoub2020model,zhou2020provably}, which assumes that the transition probability $\PP$ is a linear mixture of $d$ signed basis measures. Meanwhile, for any function $V$, we assume that we can do the summation $\sum_{s' \in \cS} \bphi(s'| s, a)V(s)$ efficiently, e.g., using Monte Carlo method \citep{yang2019reinforcement}. 
\begin{definition}[Linear Mixture MDPs \citep{jia2020model,ayoub2020model,zhou2020provably}]\label{def:mdp}
The unknown transition probability $\PP$ is a linear combination of $d$ signed basis measures $\phi_i(s'|s,a)$, i.e., $\PP(s'|s,a) = \sum_{i=1}^d \phi_i(s'|s,a)\theta^*_i$. Meanwhile, for any $V: \cS \rightarrow [0,1]$, $i \in [d], (s,a) \in \cS \times \cA$, the summation $\sum_{s' \in \cS} \phi_i(s'|s,a) V(s')$ is computable. For simplicity, let $\bphi = [\phi_1, \dots, \phi_d]^\top$, $\btheta^* = [\theta^*_1,\dots, \theta^*_d]^\top$ and $\bpsi_V(s, a) = \sum_{s' \in \cS} \bphi(s'| s, a)V(s)$. Without loss of generality, we assume $\|\btheta^*\|_2 \le B, \|\bpsi_V(s, a)\|_2 \le 1$ for all $V: \cS \rightarrow [0,1]$ and $(s,a) \in \cS \times \cA$.
\end{definition}
\begin{remark}
    A similar but notably different definition (i.e., linear MDPs \citep{yang2019sample,jin2020provably}) has  been used in~\cite{wang2020reward}, which assumes that $\PP(s'|s,a) = \la \bphi(s,a), \bmu(s')\ra$ and $r_h = \la \bphi(s,a), \btheta_h\ra$, $\bmu_h(\cdot)$ is a measure and $\btheta_h$ is an unknown vector. Comparing with linear MDPs, linear mixture MDPs do not need the reward function $r$ to be linear, which makes our algorithms more general. 
\end{remark}
With Definition~\ref{def:mdp}, it is easy to verify that the expectation of any bounded function $V$ is a linear function of $\bpsi$: 
\begin{align}
    [\PP V](s, a) = \la \bpsi_V(s, a), \btheta^*\ra.\label{eq:exp}
\end{align}

\noindent\textbf{Reward-free RL} For reward-free RL, the algorithm can be divided into two phases: \emph{exploration phase} and \emph{planning phase}. In the exploration phase, the algorithm cannot access the reward function but collect $K$ episodes by doing exploration. In the planning phase, the algorithm is given a series of reward functions and find the optimal policy based on these reward functions, using the $K$ episodes collected in the exploration phase. We formally define $(\epsilon, \delta)$-learn and sample complexity of the algorithm as follows \citep{jin2020reward}. 

\begin{definition}[$(\epsilon, \delta)$-learnability]
Given an MDP transition kernel set $\cP$, reward function set $\cR$ and a initial state distribution $\mu$, we say a reward-free algorithm can $(\epsilon, \delta)$-learn the problem $(\cP, \cR)$ with sample complexity $K(\epsilon, \delta)$, if for any transition kernel $P \in \cP$, after receiving $K(\epsilon, \delta)$ episodes in the exploration phase, for any reward function $r \in \cR$, the algorithm returns a policy $\pi$ in planning phase, such that with probability at least $1 - \delta$, $\EE_{s_1 \sim \mu} [V_1^*(s_1; r) - V_1^{\pi}(s_1; r)] \le \epsilon$.
\end{definition}

\section{Algorithm and Main Results}
In this section, we propose a reward-free algorithm. This algorithm works as follows: Firstly, during the \emph{exploration phase}, it samples the MDP episodes, build an estimator $\btheta$ for the MDP parameter $\btheta^*$, and compute the covariance matrix $\bSigma$ of the feature mappings, which characterizes the uncertainty of the estimator $\btheta$. Secondly, during the \emph{planning phase}, the algorithm uses the collected $\btheta$ and $\bSigma$ in the exploration phase to find the optimal policy $\pi$ based on the given reward functions.

\subsection{Planning phase algorithm}
We first introduce the \texttt{PLAN} function (Algorithm~\ref{alg:plan}), which is a common module in both planning phase and exploration phase. Given a series of reward functions $\{r_h\}_h$, the goal of \texttt{PLAN} function is to output the optimal policies $\{\pi_h\}_h$ and Q-functions $\{Q_h\}_h$ corresponding to $\{r_h\}_h$. Suppose the unknown parameter $\btheta^*$ is known, we can compute $\{Q_h\}_h$ recursively by the following Bellman equation:
\begin{align}
    Q_h(s, a; r) &= r_h(s, a) + [\PP V_{h+1}](s, a; r) = r_h(s, a) + \la\bpsi_{V_{h+1}}(s,a), \btheta^* \ra.\label{eq:help1}
\end{align}
However, since $\btheta^*$ is unknown, we cannot compute $Q_h$ as in \eqref{eq:help1}. Instead, \texttt{PLAN} takes the estimated parameter $\btheta$ and the ``covariance matrix'' $\bSigma$ as input. To calculate $Q_h$, \texttt{PLAN} replaces $\btheta^*$ with the estimated $\btheta$ and plus an additional exploration bonus term $\beta\|\bpsi_{V_{h+1}}(\cdot, \cdot)\|_{\bSigma^{-1}}$ to \eqref{eq:help1}, as in Line \ref{ln:ucb} of Algorithm~\ref{alg:plan}. Then \texttt{PLAN} takes the greedy policy of the calculated optimistic $Q_h$ and proceeds to the previous step. Finally, the algorithm returns policy $\pi$ in Line~\ref{ln:policy} as well as the estimated value functions $\{V_h\}_h$.



\begin{algorithm}[tb]
\caption{\algname Planning Module (\texttt{PLAN})}
\label{alg:plan}
\begin{algorithmic}[1]
\REQUIRE Estimated parameter and covariance $\btheta, \bSigma$, reward $\{r_h\}_{h=1}^H$, parameter $\beta$.
\STATE For consistency, set $Q_{H+1}(\cdot, \cdot) \leftarrow V_{H+1}(\cdot) \leftarrow 0$
\FOR{$h = H, H-1, \cdots, 1$}
\STATE Compute Q function as
\scalebox{0.95}{$Q_h(\cdot, \cdot) \leftarrow \big[r_h(\cdot, \cdot) + \big\la\bpsi_{V_{h+1}}(\cdot, \cdot), \btheta\big\ra + \beta\|\bpsi_{V_{h+1}}(\cdot, \cdot)\|_{\bSigma^{-1}}\big]_{(0, H)}$}\label{ln:ucb}
\STATE Compute value function $V_h(\cdot) \leftarrow \max_{a \in \cA}  Q_h(\cdot, a)$
\STATE Compute policy as $\pi_h(\cdot) \leftarrow \argmax_{a \in \cA}  Q_h(\cdot, a)$.\label{ln:policy}
\ENDFOR
\ENSURE Policy $\pi \leftarrow \{\pi_h\}_{h=1}^H$ and $\{V_h\}_{h=1}^H$
\end{algorithmic}
\end{algorithm}

\subsection{Exploration phase algorithm} 
Based on the introduced \texttt{PLAN} function, we propose the \algname algorithm in Algorithm~\ref{alg:explore}. In general, \algname guides the agent to explore the unknown state space without the information of the reward functions. In detail, for the $k$-th episode, \algname first defines the \emph{exploration driven reward function} as follows:
\begin{align}
     r_h^k(s, a) = \min\bigg\{1, \frac{2\beta}{H}\sqrt{\max_{f\in \cS \mapsto [0, H - h]}\|\bpsi_f(s, a)\|_{\bSigma^{-1}_{1, k}}}\bigg\},\label{eq:reward}
\end{align}
where $\bSigma_{1,k}$ is the ``covariance matrix'' of the feature mapping. Intuitively speaking, $r_h^k(s, a)$ represents the maximum possible uncertainty level of the state-action pair $(s,a)$ caused by the randomness of the MDP transition function, which is \emph{independent} of the true reward functions. Therefore, in order to obtain a good estimation of the optimal policy for any \emph{given} reward functions, it suffices to obtain the optimal policy for $r_h^k(s, a)$. Thus, after obtaining $\{r_h^k\}_h$, \algname finds the corresponding near-optimal policies $\{\pi_h^k\}_h$ using \texttt{PLAN} function, with the estimated parameter $\btheta_k$ and the ``covariance matrix'' $\bSigma_{1, k}$ as input. \algname uses $\{\pi_h^k\}_h$ as its exploration policy and observes the new episode $s_1^k, a_1^k,\dots, s_H^k, a_H^k$ induced by $\{\pi_h^k\}_h$. 

Next, \algname needs to compute the parameters $\btheta_{k+1}$ and $\bSigma_{1, k+1}$ for planning in the next episode. Similar to UCRL-VTR proposed by \citep{jia2020model, ayoub2020model}, \algname also uses a ``value-targeted regression (VTR)" estimator, which computes $\btheta_{k+1}$ as the minimizer to a ridge regression problem with the target being the past value functions. The main difference between \algname and UCRL-VTR is that, due to the lack of true reward functions, \algname can not use the estimated value functions as its regression targets. Instead, \algname defines the following \emph{pseudo value function} $u_h^k$: 
\begin{align}
    u_h^k = \argmax_{f \in \cS \mapsto [0, H - h]} \bpsi_f^\top(s_h^k, a_h^k)\bSigma^{-1}_{1, k}\bpsi_f(s_h^k, a_h^k).\label{eq:maxv}
\end{align}
Here, $u_h^k$ maximizes the ``uncertainty" caused by the transition kernel, which will help the agent to explore the state space. Now given the pseudo value functions, Algorithm~\ref{alg:explore} computes the estimated $\btheta_{k+1}$ as the minimizer to the following ridge regression problem:
\begin{align}
    \btheta_{k+1}&\leftarrow \argmin_{\btheta} \lambda\|\btheta\|_2^2 + \sum_{k'=1}^{k}\sum_{h=1}^H\Big(\big\la\btheta, \bpsi_{u_h^{k'}}(s_h^{k'}, a_h^{k'}) \big\ra - u_h^{k'}(s_{h+1}^{k'})\Big)^2,\label{eq:exptheta}
\end{align}
which has a closed-form solution as in Line~\ref{ln:reg-end}. It also updates the covariance matrix $\bSigma_{1, k+1}$ as in Line \ref{ln:reg-end}, by the observed feature mapping $\{\bpsi_{u_h^k}(s_h^k, a_h^k)\}_h$ in the current episode. In the end, after collecting $HK$ state-action samples, \algname calculates the policy $\{\pi_h\}$ as output based on $\btheta_{K+1}$ and $\bSigma_{1, K+1}$. 

\begin{algorithm}[!ht]
\caption{\algname (Hoeffding Bonus)}
\label{alg:explore}
\begin{algorithmic}[1]
\REQUIRE Confident parameter $\beta$, regularization parameter $\lambda$
\STATE \textbf{Phase I: Exploration Phase}
\STATE Initialize $\bSigma_{1, 1} \leftarrow \lambda\Ib, \bbb_1 \leftarrow \btheta_1 \leftarrow \zero$
\FOR{$k = 1, 2, \cdots, K$}
\STATE Compute the exploration driven reward function  $\{r_h^k(\cdot, \cdot)\}_{h=1}^H$ according to \eqref{eq:reward}\label{ln:reward-calc}
\STATE Compute exploration policy and value function as
\scalebox{0.78}{$
\displaystyle
    (\{\pi_h^k\}_{h=1}^H, \{V_h^k\}_{h=1}^H) \leftarrow  \text{\texttt{PLAN}} (\btheta_k, \bSigma_{1, k}, \{r_h^k\}_{h=1}^H, \beta)
$}
\STATE Receive the initial state $s_1^k \sim \mu$
\FOR{$h = 1, 2, \cdots, H$}
\STATE Take action $a_h^k \leftarrow  \pi_h^k(s_h^k)$ and receive $s_{h+1}^k$
\STATE Calculate $u_h^k$ for $s_h^k, a_h^k$ according to~\eqref{eq:maxv}\label{ln:calc-u}
\STATE Set $\bSigma_{h + 1, k} \leftarrow \bSigma_{h, k} + \bpsi_{u_h^k}(s_h^k, a_h^k)\bpsi_{u_h^k}(s_h^k, a_h^k)^\top, \bbb_{h+1, k} \leftarrow \bbb_{h, k} +\bpsi_{u_h^k}(s_h^k, a_h^k)u_h^k(s_{h+1}^k)$
\ENDFOR
\STATE Set $\bSigma_{1, k + 1} \leftarrow \bSigma_{H + 1, k}$, $\bbb_{1, k + 1} \leftarrow \bbb_{H + 1, k}, \btheta_{k+1} \leftarrow \bSigma_{1, k+ 1}^{-1}\bbb_{1, k+1}$\label{ln:reg-end}
\ENDFOR
\STATE \textbf{Phase II: Planning Phase}
\STATE Receive target reward function $\{r_h\}_{h=1}^H$
\STATE Compute policy as $
\displaystyle (\{\pi_h\}_{h=1}^H, \{V_h\}_{h=1}^H)\leftarrow  \texttt{PLAN}(\btheta_{K + 1}, \bSigma_{1, K + 1}, \{r_h\}_{h=1}^H, \beta)
$
\ENSURE Policy $\{\pi_h\}_{h=1}^H$
\end{algorithmic}
\end{algorithm}

\begin{remark}
Here we do a comparison between our \algname and the reward-free RL algorithm in \cite{wang2020reward}. The main difference is that \citet{wang2020reward} estimates $\btheta_k$ by regression with value function $V_h^k$ being the target, while our \algname does regression with the pseudo value function $u_h^k$ being the target. That is mainly due to the different problem settings (linear MDP v.s. linear mixture MDP). 


\end{remark}

\subsection{Implementation details}
In general, solving the maximization problem \eqref{eq:maxv} is hard. 
Here, we provide a simple approximate solution to the problem~\eqref{eq:reward} and~\eqref{eq:maxv} for the finite state space case ($|\cS| < \infty$). Instead of maximizing the $\ell_2$ norm-based objective $\big\|\bSigma^{-1/2}_{1, k}\bpsi_f(s_h^k, a_h^k)\big\|_2$, we write $\bpsi_f(s, a) = \bPhi(s, a)\fb$ with $\bPhi(s, a) = ( \bphi(s, a, S_1), \cdots, \bphi(s, a, S_{|\cS|}))$ and $\fb = (f(S_1), \cdots , f(S_{|\cS|}))^\top$, relax the $\ell_2$ norm into $\ell_1$ norm since $\|\xb\|_2 \ge \|\xb_1\|_1 /\sqrt{d}$ for any $\xb \in \RR^{d}$, and maximize the following $\ell_1$ norm-based objective
\begin{align}\label{eq:L1}
    \max_{\fb} \big\|\bSigma_{1, k}^{-1/2} \bPhi(s, a) \fb\big\|_1 \  \text{ subject to } \|\fb\|_\infty \le H - h.
\end{align}
\eqref{eq:L1} can be formulated as a linear programming, which can be solved by interior method~\citep{karmarkar1984new} or simplex method~\citep{dantzig1965linear} efficiently. 
Since $\|\xb\|_1/\sqrt{d} \le \|\xb\|_2 \le \|\xb\|_1$, the performance of this approximate solution is guaranteed. 
For the case where the state space is infinite, we can use state aggregation methods such as soft state aggregation \citep{michael1995reinforcement} to reduce the infinite state space to finite state space and then apply the above approximate solution to solve it.  

\subsection{Sample complexity}
Now we provide the sample complexity for Algorithm~\ref{alg:explore}.
\begin{theorem}[Sample complexity of \algname]\label{thm:h}
For Algorithm~\ref{alg:explore}, setting parameter $\beta = H\sqrt{d\log(3(1 + KH^3B^2)/\delta)} + 1,\quad \lambda = B^{-2}$, then for any $0<\epsilon < 1$, if $K = \tilde \cO(H^5d^2\epsilon^{-2})$, we have with probability at least $1 - \delta$ that, $\EE_{s \sim \mu} [V_1^*(s; r) - V_1^\pi(s; r)] \le \epsilon$.
\end{theorem}

\begin{remark}
Theorem \ref{thm:h} shows that \algname only needs $\text{poly}(d, H, \epsilon^{-1})$ sample complexity to find an $\epsilon$-optimal policy, which suggests that model-based reward-free algorithm is sample-efficient. Thanks to linear function approximation, the sample complexity only depends on the dimension of the feature mapping $d$ and the length of the episode and does not depend on the cardinalities of the state and action spaces.
\end{remark}

\begin{corollary}\label{col:h}
Under the same conditions as in Theorem \ref{thm:h}, if solving the relaxed optimization problem in \eqref{eq:L1}, Algorithm \ref{alg:explore} has $K = \tilde \cO(H^5d^3\epsilon^{-2})$ sample complexity.
\end{corollary}

\section{Improved Algorithm with Bernstein Bonus}
Theorem \ref{thm:h} suggests that \algname in Algorithm \ref{alg:explore} enjoys an $\tilde\cO(H^5d^2\epsilon^{-2})$ sample complexity to find an $\epsilon$-optimal policy. In this section, we seek to further improve the sample complexity.  

A key observation is that for any given reward functions $\{r_h\}_h$, the error between the exploration policy $\{\pi_h\}_h$ and the optimal policy can be decomposed into two parts: \emph{the exploration error} which is the difference between $\{r_h\}_h$ and the exploration driven reward function $\{r_h^k\}_h$, and \emph{the approximation error} which is the difference between the optimal value function $V_1^*(\cdot; r_h^k)$ and our estimated value function $V_1^{\pi_h^k}(\cdot; r_h^k)$ with respect to $\{r_h^k\}_h$. For the latter one, our exploration strategy adapted from VTR is often too conservative since it does not distinguish different value functions and state-action pairs from different episodes and steps. Therefore, inspired by~\citep{zhou2020nearly}, we propose a variant of \algname called \newalgname, which adopts a Bernstein-type bonus for exploration and achieves a better sample complexity.


\begin{algorithm}[!ht]
\caption{\newalgname (Bernstein Bonus)}
\label{alg:explore-bern}
\begin{algorithmic}[1]
\REQUIRE Parameter $\beta, \hat \beta, \tilde \beta, \check \beta$, regularization parameter $\lambda$
\STATE \textbf{Stage I: Exploration Phase}
\STATE Initialize $\bSigma_{1, 1} = \hat \bSigma_{1, 1} = \tilde \bSigma_{1, 1} =  \lambda\Ib, \bbb_1 = \hat \bbb_1 = \tilde \bbb_1 = \btheta_1 = \hat \btheta_1 = \tilde \btheta_1 =  \zero$
\FOR{$k = 1, 2, \cdots, K$}
\STATE Set $\{r_h^k(\cdot, \cdot)\}_{h=1}^H$ to \eqref{eq:reward}.\label{ln:reward-calc-bern}
\STATE Compute exploration policy and value function as
\scalebox{0.78}{$
\displaystyle (\{\pi_h^k\}_{h=1}^H, \{V_h^k\}_{h=1}^H) \leftarrow  \text{\texttt{PLAN}} (\hat \btheta_k, \hat \bSigma_{1, k}, \{r_h^k\}_{h=1}^H, \hat \beta)
$}\label{ln:plan}
\STATE Receive the initial state $s_1^k \sim \mu$.
\FOR{$h = 1, 2, \cdots, H$}
\STATE Take action $a_h^k =  \pi_h^k(s_h^k)$ and receive $s_{h+1}^k$\label{ln:exe}
\STATE Calculate $u_h^k, \nu_h^k$ for $s_h^k, a_h^k$ according to~\eqref{eq:maxv} and ~\eqref{eq:estv} separately\label{ln:estv}
\STATE Set $\bSigma_{h + 1, k} \leftarrow \bSigma_{h, k} + \bpsi_{u_h^k}(s_h^k, a_h^k)\bpsi_{u_h^k}(s_h^k, a_h^k)^\top$
\STATE Set $\hat \bSigma_{h + 1, k}$, $\tilde \bSigma_{h + 1, k}$, $\hat \bbb_{h + 1, k}$, $\tilde \bbb_{h + 1, k}$ using~\eqref{eq:reg-bern}
\ENDFOR
\STATE Set $\bSigma_{1, k+1} \leftarrow \bSigma_{H + 1, k}$
\STATE Set $\hat \bSigma_{1, k+1} \leftarrow \hat \bSigma_{H + 1, k}$, $\hat \bbb_{1, k+1} \leftarrow \hat \bbb_{H + 1, k}, \hat \btheta_{k+1} \leftarrow \hat \bSigma_{1, k+1}^{-1} \hat \bbb_{1, k+1}$\label{line:start}
\STATE Set $\tilde \bSigma_{1, k+1} \leftarrow \tilde \bSigma_{H + 1, k}$, $\tilde \bbb_{1, k+1} \leftarrow \tilde \bbb_{H + 1, k}, \tilde \btheta_{k+1} \leftarrow \tilde \bSigma_{1, k+1}^{-1} \tilde \bbb_{1, k+1}$\label{line:end}
\ENDFOR
\STATE Set $\btheta_{K + 1} \leftarrow \bSigma_{1, K + 1}^{-1}\sum_{k=1}^K\sum_{h=1}^H\bpsi_{u_h^k}(s_h^k, a_h^k)u_h^k(s_{h+1}^k)$
\STATE \textbf{Stage II: Planning Phase}
\STATE Receive target reward function $\{r_h\}_{h=1}^H$
\STATE Compute exploration policy as $
(\{\pi_h\}_{h=1}^H, \{V_h\}_{h=1}^H)
\leftarrow  \text{\texttt{PLAN}} (\btheta_{K+1}, \bSigma_{1, K+1}, \{r_h\}_{h=1}^H,  \beta)
$
\ENSURE Policy $\{\pi_h\}_{h=1}^H$
\end{algorithmic}
\end{algorithm}

\subsection{Exploration phase algorithm with Bernstein bonus}

\newalgname in presented in Algorithm~\ref{alg:explore-bern}. The algorithm structure is similar to that of \algname, which can be decomposed into the exploration phase and planning phase. There are two main differences. First, in contrast to \algname which uses $\btheta_k$ for the \texttt{PLAN} function in both exploration and planning phases, \newalgname only uses 
$\btheta_{K+1}$ for the \texttt{PLAN} function in the planning phase. 
For the exploration phase, \newalgname constructs a new estimator $\hat\btheta_k$ based on $\{V_{h+1}^{k'}\}_{k'\le k-1,h}$, which are the value functions of the exploration driven rewards. 
Second, to build $\hat\btheta_k$, one way is to choose it as the solution to the ridge regression problem with contexts $\bpsi_{V_{h+1}^{k'}}(s_h^{k'}, a_h^{k'})$ and targets $V_{h+1}^{k'}(s_{h+1}^{k'})$, similar to \eqref{eq:exptheta}. However, since the targets $V_{h+1}^{k'}(s_{h+1}^{k'})$ have different variances at different steps and episodes, we are actually facing a \emph{heteroscedastic linear regression} problem. Therefore, inspired by a recent line of work \cite{zhou2020nearly, wu2021nearly} which use Bernstein inequality for vector-valued self-normalized martingale to construct a tighter confidence ball for exploration, we also incorporate the variance to build choose $\hat\btheta_k$ as the solution to the following \emph{weighted ridge regression} problem, which is an enhanced estimator for the heteroscedastic case:
\begin{align}
    \hat\btheta_{k}\leftarrow \argmin_{\btheta}\lambda\|\btheta\|_2^2+\sum_{k'=1}^{k-1}\sum_{h=1}^H\Big(\big\la\btheta, \bpsi_{V_{h+1}^{k'}}(s_h^{k'}, a_h^{k'}) \big\ra - V_{h+1}^{k'}(s_{h+1}^{k'})\Big)^2/[\sigma_h^{k'}]^2, \label{help:111}
\end{align}
where $[\sigma_h^{k'}]^2$ is the variance of $V_{h+1}^{k'}(s_{h+1}^{k'})$. The idea to use variances to improve the sample complexity is closely related to the use of ``Bernstein bonus" in reward-free RL for the tabular MDPs \citep{kaufmann2020adaptive, zhang2020nearly, menard2020fast}. Since $\sigma_h^{k'}$ is unknown, we will use $\nu_h^{k'} = [\bar \sigma_h^{k'}]^2$ as a plug-in estimator to replace $[\sigma_h^{k'}]^2$ in \eqref{help:111}. After obtaining $\hat\btheta_{k}$, \newalgname sets the $\hat\bSigma_{1,k}$ as the covariance matrix of the features $\bpsi_{V_{h+1}^k}(s_h^k, a_h^k)/\bar \sigma_h^k$, and feeds it into the \texttt{PLAN} function with the exploration-driven reward functions and the confidence radius $\hat\beta$. \newalgname takes the output $\{\pi_h^k\}_h$ as the exploration policy, and $\{V_h^k\}_h$ as the value functions to construct the estimator $\hat\btheta_{k+1}$ for next episode. In the end, when it comes to the planning phase, after receiving reward functions $\{r_h\}_h$, \newalgname takes $\btheta_{K+1}$ as the solution to the ridge regression problem with contexts $\{\bpsi_{u_h^k}(s_h^k, a_h^k)\}_{k, h}$ and targets $\{u_h^k(s_{h+1}^k)\}_{k, h}$, and the covariance matrix $\bSigma_{1, K+1}$ as input, and uses $\texttt{PLAN}$ to find the near optimal policy $\{\pi_h\}_h$ with confidence radius $\beta$. It remains to specify $\nu_h^k$ in the weighted ridge regression. On the one hand, we need $\nu_h^k$ to be an upper bound of $[\sigma_h^k]^2$. On the other hand, we require $\nu_h^k$ to have a strictly positive lower bound to let \eqref{help:111} be valid. Therefore, we construct $\nu_h^k$ as follows:
\begin{align}
    \nu_h^k &= \max\{\alpha, \bar \VV_h^k(s_h^k, a_h^k)+ E_k^h(s_h^k, a_h^k)\},\label{eq:estv}
\end{align}
where $\bar \VV_h^k$ is the estimated variance of value function $V_h^k$ and $E_h^k$ is a correction term to calibrate the estimated variance, and $\alpha>0$ is a positive constant. To compute $\bar \VV_h^k(s_h^k, a_h^k)$, considering the following fact:
\begin{align}
    [\VV V_{h+1}^k] (s, a) &= [\PP [V_{h+1}^k]^2](s, a) - [\PP V_{h+1}^k](s, a)^2 = \la \btheta^*, \bpsi_{[V_{h+1}^k]^2}(s, a) \ra - \la\btheta^*,  \bpsi_{V_{h+1}^k}(s, a)\ra^2,\notag
\end{align}
it suffices to estimate $\la \btheta^*, \bpsi_{[V_{h+1}^k]^2}(s, a) \ra$ and $\la \btheta^*, \bpsi_{V_{h+1}^k}(s, a)\ra$ separately. For the first term, $\btheta^*$ can be regarded as the unknown parameter of a regression problem w.r.t. contexts $\bpsi_{[V_{h+1}^{k'}]^2}(s_h^{k'}, a_h^{k'})$ and targets $\bpsi_{[V_{h+1}^{k'}]^2}(s_h^{k'}, a_h^{k'})$. Therefore, the first term can be estimated by $\big\la \bpsi_{[V_{h+1}^k]^2}(s, a), \tilde \btheta_k\big\ra$, where
\begin{align}
    \tilde\btheta_{k}\leftarrow \argmin_{\btheta} \lambda\|\btheta\|_2^2 +\sum_{k'=1}^{k-1}\sum_{h=1}^H\Big(\big\la\btheta, \bpsi_{[V_{h+1}^{k'}]^2}(s_h^{k'}, a_h^{k'}) \big\ra - [V_{h+1}^{k'}(s_{h+1}^{k'})]^2\Big)^2. \notag
\end{align}
In addition, the second term $\la \btheta^*,  \bpsi_{V_{h+1}^k}(s, a)\ra$ can be approximated by $\la \bpsi_{V_{h+1}^k}(s, a), \hat \btheta_k\ra$. Therefore, the final estimator $[\bar\VV V_{h+1}^k] (s, a)$ is defined as
\begin{align}
    \bar \VV_h^k(s, a) = \Big[\big\la \bpsi_{[V_{h+1}^k]^2}(s, a), \tilde \btheta_k\big\ra\Big]_{(0, H^2)} - \Big[\big\la \bpsi_{V_{h+1}^k}(s, a), \hat \btheta_k\big \ra \Big]^2_{(0, H)}.\label{eq:estvar} 
\end{align}
For the correction terms $E_h^k$, we define it as follows:
\begin{align}
    E_h^k(s, a) = \min\Big\{H^2, \tilde \beta\big\|\bpsi_{[V_{h+1}^k]^2}(s, a)\big\|_{\tilde \bSigma_{1, k}^{-1}}\Big\} + \min\Big\{H^2, 2H\check \beta \big\|\bpsi_{V_{h+1}^k}(s, a)\Big\|_{\hat \bSigma_{1, k}^{-1}}\Big\}, \notag
\end{align}
where $\tilde\bSigma_{1,k}$ is the covariance matrix of the features $\bpsi_{[V_{h+1}^{k'}]^2}(s_h^{k'}, a_h^{k'})$, $\tilde\beta$, $\check\beta$ are two confidence radius. It can be shown that, with these definitions, $\bar \VV_h^k(s, a)+E_h^k(s, a)$ is an upper bound of $[\sigma_h^k]^2$.

Finally, to enable online update, \newalgname updates its covariance matrices recursively as follows, along with sequences $\hat\bbb_h^k, \tilde\bbb_h^k$: 
\begin{align}
    \hat \bSigma_{h + 1, k} &\leftarrow \hat \bSigma_{h, k} +  \bpsi_{V_{h+1}^k}(s_h^k, a_h^k)\bpsi_{V_{h+1}^k}(s_h^k, a_h^k)^\top / \nu_h^k\notag \\
    \tilde \bSigma_{h + 1, k} &\leftarrow \tilde \bSigma_{h, k} + \bpsi_{[V_{h+1}^k]^2}(s_h^k, a_h^k)\bpsi_{[V_{h+1}^k]^2}(s_h^k, a_h^k)^\top\notag \\
    \hat \bbb_{h + 1, k} &\leftarrow \hat \bbb_{h, k} + \bpsi_{V_{h+1}^k}(s_h^k, a_h^k)V_{h+1}^k(s_{h+1}^k) / \nu_h^k\notag \\
    \tilde \bbb_{h + 1, k} &\leftarrow \tilde \bbb_{h, k} + \bpsi_{[V_{h+1}^k]^2}(s_h^k, a_h^k)[V_{h+1}^k(s_{h+1}^k)]^2,\label{eq:reg-bern}
\end{align}
where $u_h^k$ is the pseudo value function in~\eqref{eq:maxv} and $\nu_h^k$ is defined in~\eqref{eq:estv}. Then \newalgname computes $\hat\btheta_k, \tilde\btheta_k$ as in Line \ref{line:start} to Line \ref{line:end} of Algorithm \ref{alg:explore-bern}. 


\subsection{Sample complexity}
Now we present the sample complexity for Algorithm~\ref{alg:explore-bern}. 

\begin{theorem}[Sample complexity of \newalgname]\label{thm:b}
For Algorithm~\ref{alg:explore-bern}, setting $\lambda = B^{-2}$, $\alpha = H^2/d$ in \eqref{eq:estv}, and the confidence radius as
\begin{align*}
    \hat \beta &= 8\sqrt{d\log(1 + KHB^2)\log(48K^2H^2/\delta)} + 4\sqrt{d}\log(48K^2H^2/\delta) + 1\\
    \check \beta &= 8d\sqrt{\log(1 + KHB^2)\log(48K^2H^2/\delta)} + 4\sqrt{d}\log(48K^2H^2/\delta) + 1\\
    \tilde \beta &= 8H^2\sqrt{d\log(1 + KHB^2)\log(48K^2H^2/\delta)} + 4H^2\log(48K^2H^2/\delta) + 1\\
    \beta &= H\sqrt{d\log(12(1 + KH^3B^2) / \delta)} + 1,
\end{align*}
then for any $0<\epsilon<1$, if $K = \tilde \cO(H^4d(H + d)\epsilon^{-2})$, then with probability at least $1 - \delta$, we have $\EE_{s \sim \mu} [V_1^*(s; r) - V_1^\pi(s; r)] \le \epsilon$.
\end{theorem}

\begin{remark}
Theorem \ref{thm:b} suggests that when $d \ge H$, the sample complexity of \newalgname is $\tilde \cO(H^4d^2\epsilon^{-2})$, which improves the sample complexity of \algname by a factor of $H$. On the other hand, when $H \ge d$, the sample complexity of \newalgname reduces to $\tilde \cO(H^5d\epsilon^{-2})$, which is better than that of \algname by a factor of $d$. At a high-level, the sample complexity improvement is attributed to the Bernstein-type bonus.
\end{remark}
\begin{corollary}\label{col:b}
Under the same conditions as in Theorem \ref{thm:b}, if solving the relaxed optimization problem in \eqref{eq:L1}, 
Algorithm \ref{alg:explore-bern} has $K = \tilde \cO(H^5d^3\epsilon^{-2})$ sample complexity.
\end{corollary}

\section{Lower Bound of Sample Complexity}

In this section, we will provide a lower bound of sample complexity for reward-free RL under linear mixture MDP setting. 


\setlength{\intextsep}{0mm}
\begin{wrapfigure}[16]{r}{0.5\textwidth}
  \centering
\begin{tikzpicture}[node distance=1.5cm,>=stealth',bend angle=45,auto]
\tikzstyle{place}=[circle,thick,draw=blue!75,fill=blue!20,minimum size=10mm]
\begin{scope}
\node [place, label] (c0){$S_1$};
\coordinate [above right of = c0, label = center:{$\vdots$}] (c1) {};
\coordinate [below right of = c0, label = center:{$\vdots$}] (d00) {};
\node [place] (d01) [right of=d00]{$S_{2, 2}$};
\node [place] (c21) [right of=c1]{$S_{2, 1}$};
\path[->] (c0)
    edge [in=150,out=90, densely dotted] node[above]{\footnotesize{$\frac12 + c\la \ab_1, \tilde \btheta_i\ra$}} (c21)
    edge [in=-150,out=30] node[]{} (c21)
    edge [in=-150,out=-90, densely dotted] node[below, align = center]{\footnotesize{$\frac12 - c\la \ab_1, \tilde \btheta_i\ra$}} (d01)
    edge [in=150,out=-30] node[]{} (d01)
    ;
\path[->] (d01)
    edge [in=30,out=-30, min distance = 6cm, loop] node[left] {$1$} ()
    ;
\path[->] (c21)
    edge [in=30,out=-30, min distance = 6cm, loop] node[left] {$1$} ()
    ;
\end{scope}
\end{tikzpicture}
\caption{
The transition kernel $\PP$ of the 
class of hard-to-learn linear mixture MDPs. 
The kernel $\PP$ is parameterized by $\btheta_i = (\sqrt{2}, \alpha \tilde \btheta_i^\top / \sqrt{d})^\top$ for some small $\alpha$. $c = \alpha / (\sqrt{2}d)$. The learner knows the MDP structure, but does not know the parameter $\btheta_i$ (or $\tilde \btheta_i \in \cM$).
}
\label{fig:hardmdp}
\end{wrapfigure}
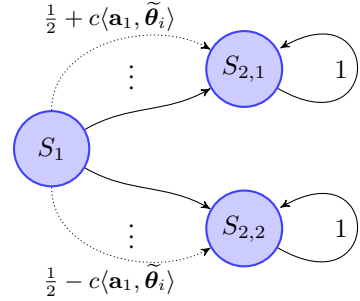
The proof is by construction. Given $d \ge 2$, we first define a binary vector set $\cM = \{\xb | \xb \in \RR^{d - 1}, [\xb]_i \in \{-1, 1\}\}$. 
We index each vector in $\cM$ as $\xb_1, \xb_2, \cdots, \xb_{|\cM|}$. Equipped with the set $\cM$, we construct a class of MDPs. As shown in Figure~\ref{fig:hardmdp}, there are in total three states $S_1, S_{2, 1}, S_{2, 2}$ and $|\cA| = |\cM|$ actions $a_1, a_2, \cdots a_{|\cA|}$. 
We define the feature mapping $\bphi(s' | s, a_i) \in \RR^d$ as follows:
\begin{align*}
    \bphi(S_{2, 1} | S_1, a_j) &= \begin{pmatrix} \frac{\sqrt{2}}4 & \frac{\ab_j^\top}{\sqrt{2d}} \end{pmatrix}^\top,\\
    \bphi(S_{2, 2} | S_1, a_j) &= \begin{pmatrix} \frac{\sqrt{2}}4 & -\frac{\ab_j^\top}{\sqrt{2d}}\end{pmatrix}^\top,\notag
\end{align*}
$\bphi(S_{2, j} | S_{2, j}, a_i) = \begin{pmatrix} 1/\sqrt{2} & \zero^\top\end{pmatrix}^\top$ for $j = 1, 2$, and $\bphi(s' | s, a) = \zero$ for all the remaining cases. 
Furthermore, we define a $d$-dimensional parameter set $\bTheta \subseteq \RR^{d+1}$ by $\bTheta = \big\{\btheta_i | \btheta_i = ( \sqrt{2}, \alpha\tilde \btheta_i^\top / \sqrt{d})^\top\big\}$ where $\tilde \btheta_i = \xb_i \in \cM$ and $\alpha$ is a small absolute constant. Therefore, for each parameter $\btheta_i$, we define the transition probability of the linear mixture MDP as $\PP(\cdot|\cdot,\cdot)=\la \bphi(\cdot|\cdot, \cdot), \btheta_i\ra$. Specifically, the transition between $S_1$ and $\{S_{2,1}, S_{2,2}\}$ is represented as
\begin{align*}
    \PP_{\btheta_i}(S_{2,1} | S_1, a_j) = \frac12 + \frac{\alpha}{\sqrt{2}d}\la\tilde \btheta_i, \ab_j\ra,\quad\PP_{\btheta_i}(S_{2,2} | S_1, a_j) = \frac12 - \frac{\alpha}{\sqrt{2}d}\la\tilde \btheta_i, \ab_j\ra.
\end{align*}
Meanwhile, we have $S_{2, 1}$ and $S_{2, 2}$ are both absorbing states. With the constructed hard-to-learn MDP class, we can prove the lower bound of sample complexity as follows:
\begin{theorem}\label{thm:mainL}
Given dimension $d \ge 50$ and $H \ge 2$, set $\epsilon \le (H - 1)/(2\sqrt{2})$ and $\delta \in (0, 1/2)$, then there exists a class of linear mixture MDPs, such that any reward-free RL algorithm that $(\epsilon, \delta)$-learns the problem $(\cP, \cR)$ where $\cR = \{\{r_h\}_{h=1}^H, 0 \le r_h \le 1\}$, needs to collect at least $K = C(1 - \delta)dH^2\epsilon^{-2}$ episodes during exploration, where $C$ is an absolute constant.
\end{theorem}
\begin{remark}\label{remark}
When $d\le H$, the sample complexity of \newalgname is $\tilde \cO(H^5d\epsilon^{-2})$, which matches the lower bound in terms of both $\epsilon$ and $d$, ignoring the logarithmic terms. When $d > H$, the sample complexity of \newalgname is $\tilde \cO(H^4d^2\epsilon^{-2})$, which matches the lower bound only in terms of $\epsilon$.  
The factor of $d$ gap between the upper and lower bounds is due to the fact that our upper bound holds for the arbitrary number of actions. Such a gap also appears in best-arm identification in the linear bandits problem (See Eq. (3) in \citet{tao2018best} with $N = \cO(2^d)$). There is also a factor of $H^2$ gap between the upper and lower bounds, and we leave it as future work to remove this gap.
\end{remark}
\section{Conclusion}

We studied model-based reward-free exploration for learning the linear mixture MDPs. We proposed two algorithms, \algname, and \newalgname, which are guaranteed to have polynomial sample complexities in exploration to find a near-optimal policy in the planning phase for any given reward function.
To our knowledge, these are the first algorithms and theoretical guarantees for model-based reward-free RL with function approximation. We also give a sample complexity lower bound for any reward-free algorithm to learn linear mixture MDPs. The sample complexity of our algorithm \newalgname matches the lower bound in terms of the dependence on accuracy $\epsilon$ and feature dimension $d$ when $H \ge d$.

\section*{Acknowledgments and Disclosure of Funding}
We thank the anonymous reviewers for their helpful comments. 
WZ, DZ and QG are partially supported by the National Science Foundation CAREER Award 1906169, IIS-1904183 and AWS Machine Learning Research Award. The views and conclusions contained in this paper are those of the authors and should not be interpreted as representing any funding agencies.

\bibliographystyle{ims}
\bibliography{refs.bib}


\section*{Checklist}

\begin{enumerate}

\item For all authors...
\begin{enumerate}
  \item Do the main claims made in the abstract and introduction accurately reflect the paper's contributions and scope?
    \answerYes{}
  \item Did you describe the limitations of your work?
    \answerYes{See Remark \ref{remark} addressing the gap between the lower bound and the upper bound}
  \item Did you discuss any potential negative societal impacts of your work?
    \answerNA{We are focusing on the theoretical analysis for Reinforcement Learning, no negative social impact can be found at this point}
  \item Have you read the ethics review guidelines and ensured that your paper conforms to them?
    \answerYes{}
\end{enumerate}

\item If you are including theoretical results...
\begin{enumerate}
  \item Did you state the full set of assumptions of all theoretical results?
    \answerYes{}
	\item Did you include complete proofs of all theoretical results?
    \answerYes{}
\end{enumerate}

\item If you ran experiments...
\begin{enumerate}
  \item Did you include the code, data, and instructions needed to reproduce the main experimental results (either in the supplemental material or as a URL)?
    \answerNA{}
  \item Did you specify all the training details (e.g., data splits, hyperparameters, how they were chosen)?
    \answerNA{}
	\item Did you report error bars (e.g., with respect to the random seed after running experiments multiple times)?
    \answerNA{}
	\item Did you include the total amount of compute and the type of resources used (e.g., type of GPUs, internal cluster, or cloud provider)?
    \answerNA{}
\end{enumerate}

\item If you are using existing assets (e.g., code, data, models) or curating/releasing new assets...
\begin{enumerate}
  \item If your work uses existing assets, did you cite the creators?
    \answerNA{}
  \item Did you mention the license of the assets?
    \answerNA{}
  \item Did you include any new assets either in the supplemental material or as a URL?
    \answerNA{}
  \item Did you discuss whether and how consent was obtained from people whose data you're using/curating?
    \answerNA{}
  \item Did you discuss whether the data you are using/curating contains personally identifiable information or offensive content?
    \answerNA{}
\end{enumerate}

\item If you used crowdsourcing or conducted research with human subjects...
\begin{enumerate}
  \item Did you include the full text of instructions given to participants and screenshots, if applicable?
    \answerNA{}
  \item Did you describe any potential participant risks, with links to Institutional Review Board (IRB) approvals, if applicable?
    \answerNA{}
  \item Did you include the estimated hourly wage paid to participants and the total amount spent on participant compensation?
    \answerNA{}
\end{enumerate}

\end{enumerate}

\newpage
\appendix

\section{Proofs of Upper Bounds}\label{sec:proofU}

In this section, we provide the proofs of sample complexity upper bounds.

\subsection{Proof of Theorem~\ref{thm:h}}
We will first introduce a lemma to show that for the planning module Algorithm~\ref{alg:plan}, if it is guaranteed that the estimation $\btheta$ is close to the true parameter $\btheta^*$, then the estimated value function is optimistic. Also the gap between the optimal value function and the value function of the output policy $\{\pi_h\}_{h=1}^H$ could be controlled by the summation of UCB bonus term. 

\begin{lemma}\label{lm:plan}
Let $\btheta, \bSigma, \beta$ be as defined in Algorithm~\ref{alg:plan}. Suppose there exists some event $\bxi$ such that $\|\btheta^* - \btheta\|_{\bSigma} \le \beta$ on this event. Then on this event, for all $s \in \cS$, $V_1(s) \ge V_1^*(s; r)$, where $V_1$ is the output value function for Algorithm~\ref{alg:plan}. We also have that
\begin{align*}
    V_1(s) - V_1^\pi(s) \le \EE\bigg[\sum_{h=1}^H \min\{H, 2\beta \|\bpsi_{V_{h+1}}(s_h, \pi_h(s_h))\|_{\bSigma^{-1}}\}\Big| s, \pi\bigg],
\end{align*}
where the policy $\pi = \{\pi_h\}_{h=1}^H$ is generated by the planning module Algorithm~\ref{alg:plan} and $V_h$ is the value function calculated on Line~\ref{ln:policy} in Algorithm~\ref{alg:plan}.
\end{lemma}

Next we will give the lemmas on how to guarantee the condition of Lemma~\ref{lm:plan} and how to utilize the result of that lemma to control the final policy error $V_1^*(s_1; r) - V_1^\pi(s_1; r)$ where the policy $\pi$ is output of the planning phase. We start with Algorithm~\ref{alg:explore}, which uses the Hoeffding bonus.

Firstly, the next lemma shows how to guarantee the condition in Lemma~\ref{lm:plan}.
\begin{lemma}[Confidence interval, Hoeffding]\label{lm:ci-h}
For Algorithm~\ref{alg:explore}, let $\lambda, \beta$ be as defined in Theorem~\ref{thm:h}, then with probability at least $1 - \delta / 3$, $\|\btheta^* - \btheta_k\|_{\bSigma_{1, k}} \le \beta$ for any $k \in [K+1]$.
\end{lemma}

Secondly, based on the lemma above, we find that the policy error during the planning phase is controlled by a summation of the UCB terms. Since from the intuition, the exploration driven reward function~\eqref{eq:reward} is the UCB term divided by $H$, the policy error during the planning phase can be converted to the value function $V_1^k$ in the exploration phase. The next lemma shows that the summation of $V_1^k$ over $K$ iterations is sub-linear to $K$, thus the policy error during the planning phase should be small.
\begin{lemma}[Summation, Hoeffding]\label{lm:sum-h}
Set the parameters of Algorithm~\ref{alg:explore} as that of Theorem~\ref{thm:h}. If the condition in Lemma~\ref{lm:ci-h} holds, then with probability at least $1 - \delta / 3$, the summation of the value function $V_1^k(s_1^k)$ during the exploration phase is controlled by 
\begin{align*}
    \sum_{k=1}^K V_1^k(s_1^k) &\le 8\beta\sqrt{HKd\log(1 + KH^3B^2 / d)} \\
    &\quad+ 8\beta H d\log(1 + KH^3B^2) + 2H\sqrt{2HK\log(1 / \delta)}.
\end{align*}
\end{lemma}
Equipped with these lemmas, we are about to prove Theorem~\ref{thm:h}.

\begin{proof}[Proof of Theorem~\ref{thm:h}]
In the following proof, we condition on the events in Lemma~\ref{lm:ci-h} and Lemma~\ref{lm:sum-h} which holds with probability at least $1 - 2\delta / 3$ by taking the union bound. Applying Lemma~\ref{lm:plan} to the final planning phase, we have
\begin{align}
    V_1^*(s; r) - V_1^\pi(s; r) \le V_1(s; r) - V_1^\pi(s; r)\le \underbrace{\EE\bigg[\sum_{h=1}^H \min\{H, 2\beta \|\bpsi_{V_{h+1}}(s_h, \pi_h(s_h))\|_{\bSigma^{-1}_{1, K + 1}}\}\bigg]}_{I_1},\label{eq:t1}
\end{align}
where the expectation is taken condition on initial state $s$ and policy $\pi$ generated by the planning phase. Since $\bSigma_{1, k} \preceq \bSigma_{1, K+1}$ for all $k \in [K]$, we can guarantee that $\|\bpsi_{V_{h+1}}(s_h, \pi_h(s_h))\|_{\bSigma^{-1}_{1, K + 1}} \le \|\bpsi_{V_{h+1}}(s_h, \pi_h(s_h))\|_{\bSigma^{-1}_{1, k}}$. Recall the exploration driven reward function is defined by
\begin{align}
     r_h^k(s, a) = \min\bigg\{1, \frac{2\beta}{H}\sqrt{\max_{f\in \cS \mapsto [0, H - h]}\|\bpsi_f(s, a)\|_{\bSigma^{-1}_{1, k}}}\bigg\},\eqref{eq:part2}
\end{align}
one can easily verify that $\min\{H, 2\beta\|\bpsi_{V_{h+1}}(s_h, \pi_h(s_h))\|_{\bSigma^{-1}_{1, k}}\} \le Hr_h^k(s_h, \pi_h(s_h))$. Therefore for any $k \in [K]$ episode, we can bound the term $I_1$ using the value function $V_1^\pi(s; \{r_h^k\}_{h=1}^H)$ of the output policy $\pi$ in the planning phase given the $\{r_h^k\}_{h=1}^H$ as the reward function, i.e.
\begin{align}
    I_1 \le \EE\bigg[\sum_{h=1}^H H r_h^k(s_h, \pi_h(s_h))\bigg] = HV_1^\pi(s; \{r_h^k\}_{h=1}^k). \label{eq:part2}
\end{align}

Plugging the bound of $I_1$ back into~\eqref{eq:t1} then taking the expectation over the initial state distribution $\mu$, we have for any $k \in [K]$,
\begin{align*}
    \EE_{s \sim \mu} [V_1^*(s; r) - V_1^\pi(s; r)] &\le  H \EE_{s \sim \mu} [V^\pi_1(s; \{r_h^k\}_{h=1}^k)] \notag \\
    &= H\Big(V^\pi_1(s_1^k; \{r_h^k\}_{h=1}^k) - V^\pi_1(s_1^k; \{r_h^k\}_{h=1}^k)\Big)\\
    &\quad +  H\EE_{s \sim \mu} [V^\pi_1(s; \{r_h^k\}_{h=1}^k)].
\end{align*}
Hence 
\begin{align}
     \EE_{s \sim \mu} [V_1^*(s; r) - V_1^\pi(s; r)] &\le \frac HK \sum_{k=1}^K\Big(V_1^\pi(s_1^k; \{r_h^k\}_{h=1}^k) - V_1^\pi(s_1^k; \{r_h^k\}_{h=1}^k) \notag  \\
     &\qquad+  \EE_{s \sim \mu} [V_1^\pi(s; \{r_h^k\}_{h=1}^k)]\Big).\label{eq:zz2}
\end{align}
Since $V^\pi_1(s; \{r_h^k\}_{h=1}^k) \le H$ for all $k \in [K], s \in \cS$, by Azuma-Hoeffding's inequality, with probability at least $1 - \delta / 3$, 
\begin{align}
    \sum_{k=1}^K\Big(\EE_{s \sim \mu} [V_1^\pi(s; \{r_h^k\}_{h=1}^k)] - V_1^\pi(s_1^k; \{r_h^k\}_{h=1}^k)\Big) \le H\sqrt{2K\log(3 / \delta)}.\label{eq:zz1}
\end{align}
By plugging~\eqref{eq:zz1} into~\eqref{eq:zz2}, we have 
\begin{align*}
    \EE_{s \sim \mu} [V_1^*(s; r) - V_1^\pi(s; r)] \le \frac HK \sum_{k=1}^KV_1^\pi(s_1^k; \{r_h^k\}_{h=1}^k) + H^2\sqrt{2\log(3 / \delta) / K}.
\end{align*}
Applying Lemma~\ref{lm:plan} to the exploration phase, for any $k$-th episode, $V^\pi_1(s_1^k; \{r_h^k\}_{h=1}^k) \le V^*_1(s_1^k; \{r_h^k\}_{h=1}^k) \le V_1^k(s_1^k)$, thus replacing the value function $V^\pi_1$ with the estimated value function $V_1^k$, we have 
\begin{align}
    \EE_{s \sim \mu} [V_1^*(s; r) - V_1^\pi(s; r)] \le \frac HK \sum_{k=1}^KV_1^k(s_1^k) +  H^2\sqrt{2\log(3 / \delta) / K}.\label{eq:f1}
\end{align}
Finally by Lemma~\ref{lm:sum-h} we can bound the summation over $V_1^k$, hence 
\begin{align*}
    \EE_{s \sim \mu} [V_1^*(s; r) - V_1^\pi(s; r)] &\le  H^2\sqrt{2\log(3 / \delta)/ K} + 8\beta\sqrt{H^3d\log(1 + KH^3B^2 / d) / K} \\
    &\quad+ 8\beta dH^2\log(1 + KH^3B^2) / K + 2H^2\sqrt{2H\log(1 / \delta) / K}
\end{align*}
and by taking union bound, the result holds with probability at least $1 - \delta$. Recall the setting of $\beta \sim \tilde \cO(H\sqrt{d})$ as in Theorem~\ref{thm:h}, let $K = \tilde \cO(H^5d^2\epsilon^{-2})$,
the policy error $\EE_{s \sim \mu} [V_1^*(s; r) - V_1^\pi(s; r)]$ is bounded by $\epsilon$.
\end{proof}

\subsection{Proof of Corollary \ref{col:h}}

\begin{proof}[Proof of Corollary \ref{col:h}]
Following the proof of Theorem \ref{thm:h}, since for all $\xb \in \RR^d, \|\xb\|_1 \le \|\xb\|_2 \le \sqrt{d}\|\xb\|_1$ it follows that
\begin{align}
    \|\bpsi_{V_{h+1}}(s_h, \pi_h(s_h))\|_{\bSigma_{1, K+1}^{-1}} &= \|\bSigma_{1, K+1}^{-1/2}\bpsi_{V_{h+1}}(s_h, \pi_h(s_h))\|_2 \notag \\
    &\le \sqrt{d} \|\bSigma_{1, K+1}^{-1/2}\bpsi_{V_{h+1}}(s_h, \pi_h(s_h))\|_1. \label{eq:part1}
\end{align}

We denote $\tilde {u}_h^k$ as the result using the $\ell_1$ norm as the surrogate objective function in this optimization problem \eqref{eq:L1}, i.e. 
\begin{align*}
    \tilde {u}_h^k := \argmax_{f \in \cS \mapsto [0, H - h]}\|\bSigma^{-1/2}_{1, k}\bpsi_f(s_h^k, a_h^k)\|_1,
\end{align*}
then \eqref{eq:part1} yields
\begin{align*}
    \|\bpsi_{V_{h+1}}(s_h, \pi_h(s_h))\|_{\bSigma_{1, K+1}^{-1}} &\le\sqrt{d} \|\bSigma_{1, K+1}^{-1/2}\bpsi_{V_{h+1}}(s_h, \pi_h(s_h))\|_1\\
    &\le \sqrt{d} \|\bSigma_{1, K+1}^{-1/2}\bpsi_{\tilde {u}_h^k}(s_h, \pi_h(s_h))\|_1\\
    &\le \sqrt{d} \|\bSigma_{1, K+1}^{-1/2}\bpsi_{\tilde {u}_h^k}(s_h, \pi_h(s_h))\|_2\\
    &\le \sqrt{d} \|\bSigma_{1, K+1}^{-1/2}\bpsi_{u_h^k}(s_h, \pi_h(s_h))\|_2,
\end{align*}
where the second inequality comes from $\tilde {u}_h^k$ is the solution in \eqref{eq:L1}, the third inequality comes from the fact that $\|\xb\|_1 \le \|\xb\|_2$ and the forth inequality comes from the definition that $u_h^k$. Then \eqref{eq:part2} is changed to be 
\begin{align*}
    I_1 \le H\sqrt{d} V_1^{\pi}(s, \{r_h^k\}_{h=1}^k).
\end{align*}
Noticing that comparing to the original result, there's an additional $\sqrt{d}$ factor which yields \eqref{eq:part1}
\begin{align*}
    \EE_{s \sim \mu} [V_1^*(s; r) - V_1^\pi(s; r)] \le \frac {H\sqrt{d}}K \sum_{k=1}^KV_1^k(s_1^k) +  H^2\sqrt{2d\log(3 / \delta) / K}.
\end{align*}
Then it is easy to show that using $\ell_1$ as the surrogate objective function, the sample complexity of Algorithm \ref{alg:explore} turns out to be $\tilde \cO(H^5d^3\epsilon^{-2})$
\end{proof}

\subsection{Proof of Theorem~\ref{thm:b}}

We are going to analyze Algorithm~\ref{alg:explore-bern} and provide the proof of Theorem~\ref{thm:b}. Following the proof of Theorem~\ref{thm:h}, we only need to revise Lemmas~\ref{lm:ci-h} and \ref{lm:sum-h} to continue the proof of Theorem~\ref{thm:b}.

\begin{lemma}[Confidence interval, Bernstein]\label{lm:ci-b}

Let $\beta, \hat \beta, \tilde \beta, \check \beta$ and $\lambda$ be defined as Theorem~\ref{thm:b}, then with probability at least $1 - \delta / 3$, for all $k \in [K + 1]$,
\begin{align}
    \|\btheta^* - \hat \btheta_k\|_{\hat \bSigma_{1, k}} \le \hat \beta, \ \|\btheta^* - \hat \btheta_k\|_{\hat \bSigma_{1, k}} \le \check \beta,\ 
    \|\btheta^* - \tilde \btheta_k\|_{\tilde \bSigma_{1, k}} \le \tilde \beta, \ \|\btheta^* - \btheta_{K + 1}\|_{ \bSigma_{1, K + 1}} \le \beta,\label{eq:bern-cond}
\end{align}
and $|[\VV_h V_{h+1}^k](s, a) - \bar \VV_h^k(s, a)| \le E_h^k(s, a)$.
\end{lemma}

\begin{lemma}[Summation, Bernstein]\label{lm:sum-b}
For Algorithm~\ref{alg:explore}, setting its parameters as in Lemma~\ref{lm:ci-h}, with probability at least $1 - \delta / 3$, the summation of the value function during exploration phase is controlled by 
\begin{align*}
    \sum_{k=1}^K V_1^k(s_1^k) \le \tilde  \cO(\sqrt{H^3Kd} + Hd\sqrt{K}) + o(\sqrt{K}).
\end{align*}
\end{lemma}
\begin{proof}[Proof of Theorem~\ref{thm:b}]
The proof is almost the same as the proof of Theorem~\ref{thm:h} by replacing Lemma~\ref{lm:ci-h} with Lemma~\ref{lm:ci-b}, Lemma~\ref{lm:sum-h} with Lemma~\ref{lm:sum-b}. In detail, following the same method, \eqref{eq:f1} works for Algorithm~\ref{alg:explore-bern} under the condition in Lemma~\ref{lm:ci-b} holds. Therefore, by using Lemma~\ref{lm:sum-b} instead of Lemma~\ref{lm:sum-h}, with probability at least $1 - \delta$,
\begin{align*}
    \EE_{s \sim \mu} [V_1^*(s; r) - V_1^\pi(s; r)]&\le \frac HK \sum_{k=1}^KV_1^k(s_1^k) +  H^2\sqrt{2\log(3 / \delta) / K} \\
    &\le \tilde \cO\Big((\sqrt{H^4d^2} + \sqrt{H^5d})/\sqrt{K}\Big).
\end{align*}
Letting $K = \tilde \cO(H^4d(H + d)\epsilon^{-2})$, the policy error for the planning phase could be controlled by $\EE_{s \sim \mu} [V_1^*(s; r) - V_1^\pi(s; r)] \le \epsilon$. 
\end{proof}
\subsection{Proof of Corollary \ref{col:b}} 
\begin{proof}[Proof of Corollary \ref{col:b}]
The proof is almost the same as proof of Corollary \ref{col:h}, by adding the additional dependency $d$ into the regret bound achieved by Theorem \ref{thm:b}, it's easy to verify that the sample complexity using the $\ell_1$ norm as the surrogate function \eqref{eq:L1} is $\tilde \cO(H^4d^2(H + d)\epsilon^{-2})$
\end{proof}
\section{Missing Proofs in Appendix~\ref{sec:proofU}}\label{sec:A}
\subsection{Filtration}\label{sec:fil}
For the simplicity of further proof, we define the event filtration here as 
\begin{align*}
    \cG_{h, k} = \big\{\{s_i^\kappa, a_i^\kappa\}_{i=1, \kappa=1}^{H, k - 1}, \{s_i^k, a_i^k\}_{i=1}^{h-1}\big\},
\end{align*}
it is easy to verify that $s_h^k$ is $\cG_{h+1, k}$-measurable. Also, since $\pi^k$ is $\cG_{h, k}$-measurable for all $h \in [H]$, $a_h^k = \pi_h^k(s_h^k)$ is also $\cG_{h+1, k}$-measurable. Also, for any function $f \le R$ built on $\cG_{h+1, k}$, such as $V_{h+1}^k, u_{h}^k$, $f(s_{h+1}^k) - [\PP f](s_h^k, a_h^k)$ is $\cG_{h+1, k}$-measurable and it is also a zero-mean $R$-sub-Gaussian conditioned on $\cG_{h+1, k}$.

Since $\cG_{H+1, k} = \cG_{1, k + 1}$, we could arrange the filtration as 
\begin{align*}
    \cG = \{\cG_{1, 1}, \cdots, \cG_{H, 1}, \cdots, \cG_{1, k}, \cdots, \cG_{h, k}, \cdots \cG_{H, k}, \cdots, \cG_{1, k+1}, \cdots, \cG_{H, K}, \cG_{1, K+1}\},
\end{align*}
and we will use $\cG$ as the filtration set for all of the proofs in the following section and it is obvious that $\cG_{1, K+1}$ contains all information we collect during the exploration phase.
\subsection{Proof of Lemma~\ref{lm:plan}}

\begin{proof}[Proof of Lemma~\ref{lm:plan}]
We prove this lemma by induction on time step $h$. Indeed, when $h = H + 1$, $V_{H+1}(s) = V^{*}_{H+1}(s; r) = 0$ by definition. Suppose for $h \in [H]$, $V_{h + 1}(s) \ge V^{*}_{h+1}(s; r)$, then following the update rule of $Q$ function in Algorithm~\ref{alg:plan}, we have 
\begin{align*}
    &Q_h(s, a) - Q^*_h(s, a; r) \\
     &\quad=  \min\big\{H, r_h(s, a) + \la\bpsi_{V_{h+1}}(s, a), \btheta \ra + \beta\|\bpsi_{V_{h+1}}(s, a)\|_{ \bSigma_{}^{-1}}\big\} - r_h(s, a)  - [\PP V^*_{h+1}](s, a; r) \\
    &\quad\ge \min\big\{H - Q_h^*(s, a; r), \la\bpsi_{V_{h+1}}(s, a), \btheta\ra + \beta\|\bpsi_{V_{h+1}}(s, a)\|_{ \bSigma^{-1}} - [\PP V^*_{h+1}](s, a; r)\big\}.
\end{align*}
We need to show that $Q_h(s, a) \ge Q_h^*(s, a ; r)$. Since it is obvious that the first term $H - Q_h^*(s, a; r)$ in $\min$ operator is greater than zero, we only need to verify that the second term is also positive where
\begin{align*}
    &\la\bpsi_{V_{h+1}}(s, a), \btheta\ra + \beta\|\bpsi_{V_{h+1}}(s, a)\|_{ \bSigma^{-1}} - [\PP V^*_{h+1}](s, a; r)\\
    &\quad\ge\la\bpsi_{V_{h+1}}(s, a), \btheta\ra + \beta\|\bpsi_{V_{h+1}}(s, a)\|_{ \bSigma^{-1}} - [\PP V_{h+1}](s, a; r)\\
    &\quad= \la\bpsi_{V_{h+1}}(s, a), \btheta - \btheta^*\ra + \beta\|\bpsi_{V_{h+1}}(s, a)\|_{ \bSigma^{-1}}\\
    &\quad\ge \beta\|\bpsi_{V_{h+1}}(s, a)\|_{ \bSigma^{-1}} - \|\bpsi_{V_{h+1}}(s, a)\|_{\bSigma^{-1}}\|\btheta - \btheta^*\|_{\bSigma},
\end{align*}
where the first inequality is from the induction assumption that $V_{h+1}^{*}(s; r) \le V_{h+1}(s)$. The second equality is from the expectation of value function is a linear function of $\bpsi_{V_{h+1}}$ shown in~\eqref{eq:exp}. Then the inequality on the third line is utilizing the fact that $\la \xb, \yb\ra \ge -\|\xb\|_{\Ab^{-1}}\|\yb\|_{\Ab}$. Since it is guaranteed that $\beta \ge \|\btheta - \btheta^*\|_{\bSigma}$ from the statement of this lemma, $Q_{h}(s, a) - Q^*_h(s, a; r) \ge 0$, which from induction we get our conclusion.

For the second part controlling $V_1(s) - V_1^\pi(s)$, since aforementioned proof has shown that $V_h^*(s; r) \le V_h(s)$ for all $h \in [H]$, we have $V^*_h(s; r) - V^\pi_h(s; r) \le V_h(s) - V^\pi_h(s; r)$ and 
\begin{align*}
    V_h(s) - V^\pi_h(s; r)
    &= \min\{H, r_h(s, \pi_h(s)) + \la\bpsi_{V_{h+1}}, \btheta\ra + \beta \|\bpsi_{V_{h+1}}(s, \pi_h(s))\|_{\bSigma^{-1}}\} \\
    &\quad- r_h(s, \pi_h(s)) - [\PP V^\pi_{h+1}](s, \pi_h(s); r)\\
    &\le \min\{H, \la\bpsi_{V_{h+1}}, \btheta\ra + \beta \|\bpsi_{V_{h+1}}(s, \pi_h(s))\|_{\bSigma^{-1}} - [\PP V_{h+1}](s, \pi_h(s))\} \\
    &\quad+ [\PP V_{h+1}](s, \pi_h(s))\} - [\PP V_{h+1}^\pi](s, \pi_h; r)\\
    &= \min\{H, \la\bpsi_{V_{h+1}}, \btheta - \btheta^*\ra + \beta \|\bpsi_{V_{h+1}}(s, \pi_h(s))\|_{\bSigma^{-1}}\} \\
    &\quad + [\PP V_{h+1}](s, \pi_h(s))\} - [\PP V_{h+1}^\pi](s, \pi_h(s); r)\\
    &\le \min\{H, 2\beta \|\bpsi_{V_{h+1}}(s, \pi_h(s))\|_{\bSigma^{-1}}\} \\
    &\quad+ [\PP V_{h+1}](s, \pi_h(s))\} - [\PP V_{h+1}^\pi](s, \pi_h(s); r),
\end{align*}
where the first inequality is directly from moving term $- r_h(s, \pi_h(s)) - [\PP V_{h+1}](s, \pi_h(s))$ into the $\min$ operator, the second inequality uses the condition that $\|\btheta - \btheta^*\|_{\bSigma} \le \beta$ and $\la \xb, \yb\ra \le \|\xb\|_{\Ab^{-1}}\|\yb\|_{\Ab}$. Considering the first step $h = 1$, we have
\begin{align*}
    V_1(s_1) - V_1^\pi(s_1; r) &\le \min\{H, 2\beta \|\bpsi_{V_{2}}(s_1, \pi_1(s_1))\|_{\bSigma^{-1}}\} + \EE_{s_2 \sim \PP(\cdot | s_1, \pi_1(s_1))}[V_2(s_2) - V_2^\pi(s_2)] \\
    &\le \min\{H, 2\beta \|\bpsi_{V_{2}}(s_1, \pi_1(s_1))\|_{\bSigma^{-1}}\} \\
    &\quad+ \EE_{s_2 \sim \PP(\cdot | s_1, \pi_1(s_1))}\Big[ \min\{H, 2\beta \|\bpsi_{V_{3}}(s_2, \pi_2(s_2))\|_{\bSigma^{-1}}\} \\
    &\quad+ \EE_{s_3 \sim \PP(\cdot | s_2, \pi_2(s_2))} [V_3(s_3) - V_3^\pi(s_3)]\Big] \\
    &\le \cdots \\
    &\le \EE\bigg[\sum_{h=1}^H \min\{H, 2\beta \|\bpsi_{V_{h+1}}(s_h, \pi_h(s_h))\|_{\bSigma^{-1}}\}\bigg| s_1, \pi\bigg],
\end{align*}
which concludes our proof.
\end{proof}


\subsection{Proof of Lemma~\ref{lm:ci-h}}
We introduce the classical confidence set lemma from~\citep{abbasi2011improved}.
\begin{lemma}[Theorem 2,~\cite{abbasi2011improved}]\label{lm:abbasithm2}
Let $\{\cF_t\}_{t=0}^\infty$ be a filtration and $\{\eta_t\}$ is a real-valued stochastic process which is $F_t$-measurable and conditionally $R$-sub-Gaussian. Set $y_t = \la \xb_t, \bpsi^*\ra + \eta_t$, $\Vb_t = \lambda\Ib + \sum_{i=1}^t \xb_i\xb_i^\top$ where $\xb \in \RR^d$. Denote the estimation of $\bpsi^*$ as $\bpsi_t = \Vb_t^{-1}\sum_{i=1}^t y_i\xb_i$. If $\|\bpsi^*\|_2 \le S, \|\xb_t\|_2 \le L$, then with probability at least $1 - \delta$, for all $t \ge 0$
\begin{align*}
    \|\bpsi^* - \bpsi_t\|_{\Vb_t} \le R\sqrt{d\log\bigg(\frac{1 + tL^2 / \lambda}{\delta}\bigg)} + S\sqrt{\lambda}.
\end{align*}
\end{lemma}

Equipped with this lemma, we begin our proof.
\begin{proof}[Proof of Lemma~\ref{lm:ci-h}]
Since $[\PP u_h^k](s_h^k, a_h^k) = \la\bpsi_{u_h^k}(s_h^k, a_h^k), \btheta^*\ra$ due to~\eqref{eq:exp} and $u_h^k(s) \le H$ , $u_h^k(s) - \la\bpsi_{u_h^k}(s_h^k, a_h^k), \btheta^*\ra$ is $\cG_{h,k}$-measurable and it is also a zero mean $H$-sub-Gaussian random variable conditioned on $\cG_{h,k}$. Also from Definition~\ref{def:mdp}, $\|\btheta^*\|_2 \le B, \|\bpsi_{u_{h}^k}(s_h^k, a_h^k)\|_2 \le H$.
Therefore, recall the calculation of $\btheta_k$, according to Lemma~\ref{lm:abbasithm2}, let $t = (k - 1)H$ we have
\begin{align*}
    \|\btheta_k - \btheta^*\|_{\bSigma_{1, k}} \le H\sqrt{d\log\bigg(\frac{1 + (k - 1)H^3 / \lambda}{\delta}\bigg)} + B\sqrt{\lambda}.
\end{align*}
Let $\lambda = B^{-2}$, $\delta = \delta / 3$ and relax $k$ with $k = K + 1$, we can get the $\beta$ claimed in Theorem~\ref{thm:h}.
\end{proof}

\subsection{Proof of Lemma~\ref{lm:sum-h}}
We provide the proof to control the summation of the value function during the exploration phase. To start with, since rather than immediately updating the parameter after each time step, we can only update the estimation $\btheta$ and its `covariance matrix' $\bSigma$ once after each episode. As a result, this `batched update rule' make the UCB bonus term at step $(h, k)$ be $\|\bpsi_{u_h^k}(s_h^k, a_h^k)\|_{\Ub_{1, k}^{-1}}$ instead of $\|\bpsi_{u_h^k}(s_h^k, a_h^k)\|_{\Ub_{h, k}^{-1}}$ in the vanilla linear bandit setting. Therefore, we need lemmas showing that these two UCB terms are close to each other.


\begin{lemma}\label{lm:batch}
For any $\{\xb_{h, k}\}_{h=1, k=1}^{H, K} \subset \RR^d$ satisfying that $\|\xb_{h, k}\|_2 \le L, \forall (h, k) \in [H]\times [K]$, let $\Ub_{h, k} = \lambda \Ib + \sum_{\kappa=1}^{k-1}\sum_{i=1}^H \xb_{i, \kappa}\xb_{i, \kappa}^\top + \sum_{i=1}^{h-1}\xb_{i, k}\xb_{i, k}^\top$, there exists at most $2Hd\log(1 + KHL^2/\lambda)$ pairs of $(h, k)$ tuple such that $\det \Ub_{h, k} \le 2\det \Ub_{1, k}$.
\end{lemma}

\begin{lemma}[Lemma 12,~\cite{abbasi2011improved}]\label{lm:det}
Suppose $\Ab, \Bb \in \RR^{d\times d}$ are two positive definite matrices satisfying that $\Ab \succeq \Bb$, then for any $\xb \in \RR^d$, we have $\|\xb\|_{\Ab} \le \|\xb\|_{\Bb}\sqrt{\det(\Ab) / \det(\Bb)}$.
\end{lemma}

Following that, we also need to introduce the classical lemma to control the summation of the UCB bonus terms in vanilla linear bandit setting.

\begin{lemma}[Lemma 11,~\cite{abbasi2011improved}]\label{lm:11}
For any $\{\xb_t\}_{t=1}^T \subset \RR^d$ satisfying that $\|\xb_t\|_2 \le L, \forall t \in [T]$, let $\Ub_t = \lambda \Ib + \sum_{\tau = 1}^{t - 1} \xb_\tau \xb_\tau^\top$, we have 
\begin{align}
    \sum_{t=1}^T \min\{1, \|\xb_t\|_{\Ub^{-1}_t}\}^2 \le 2d\log\bigg(\frac{d\lambda + TL^2}{d\lambda}\bigg).\notag
\end{align}
\end{lemma}

We also need to introduce the Azuma-Hoeffding's inequality to build the concentration bound for martingale difference sequences.
\begin{lemma}[Azuma-Hoeffding's inequality,~\cite{azuma1967weighted}]\label{lm:azuma}
Let $\{x_i\}_{i=1}^n$ be a martingale difference sequence with respect to a filtration $\{\cG_i\}_{i=1}^n$ (i.e. $\EE[x_i | \cG_i] = 0$ a.s. and $x_i$ is $\cG_{i+1}$ measurable) such that $|x_i| \le M$ a.s.. Then for any $0 < \delta < 1$, with probability at least $1 - \delta$, $\sum_{i=1}^n x_i \le M\sqrt{2n\log(1 / \delta)}$.
\end{lemma}
\begin{proof}[Proof of Lemma~\ref{lm:sum-h}]
By Lemma~\ref{lm:plan}, for the $k$-th episode, we have 
\begin{align}
    V_1^k(s_1^k) - V^{\pi^k}(s_1^k) &= \EE\bigg[\sum_{h=1}^H \min\{H, 2\beta \|\bpsi_{V_{h+1}^k}(s_h, \pi_h^k(s_h))\|_{\bSigma_{1, k}^{-1}}\}\bigg| s_1^k, \pi^k\bigg]\notag \\
    &\le \EE\bigg[\sum_{h=1}^H \min\{H, 2\beta \|\bpsi_{u_{h}^k}(s_h, \pi_h^k(s_h))\|_{\bSigma_{1, k}^{-1}}\}\bigg| s_1^k, \pi^k\bigg]\label{eq:e1}
\end{align}
where the inequality comes from that the pseudo value function $u_h^k$ defined in~\eqref{eq:maxv} is from maximizing the UCB term $\|\bpsi_{V_{h+1}^k}(s_h, \pi_h^k(s_h))\|_{\bSigma_{1, k}^{-1}}$ and we denote $\{\pi_h^k\}_{h=1}^H$ by $\pi^k$ in short. By the definition of $r_h^k$, we have
\begin{align}
    V^{\pi^k}(s_1^k) &= \EE[\sum_{h=1}^H r_h^k(s_h, \pi_h^k(s_h))| s_1^k, \pi^k] \notag \\
    &= \EE\bigg[\sum_{h=1}^H \min\{1, 2\beta \|\bpsi_{u_{h}^k}(s_h, \pi_h^k(s_h))\|_{\bSigma_{1, k}^{-1}} / H \}\bigg| s_1^k, \pi^k\bigg]\label{eq:e2}.
\end{align}
Adding~\eqref{eq:e1} and~\eqref{eq:e2} together and taking summation over $k$, we have 
\begin{align}
    \sum_{k=1}^K V_1^k(s_1^k) \le \frac{H + 1}{H}\underbrace{\sum_{k=1}^K\EE\bigg[\sum_{h=1}^H \min\{H, 2\beta \|\bpsi_{u_{h}^k}(s_h, \pi_h^k(s_h))\|_{\bSigma_{1, k}^{-1}}\}\bigg| s_1^k, \pi^k\bigg]}_{I_1} \le 2I_1,\label{eq:total}
\end{align}
where the last inequality is due to $(H + 1) / H \le 2$. 
Next we are going to control the expectation of summation $I_1$. Consider the filtration $\{\cG_{h, k}\}_{h=1, k=1}^{H, K}$ defined in Section~\ref{sec:fil}, denote $x_{h,k}$ as follows:
\begin{align*}
    x_{h, k} = \min\{H, 2\beta \|\bpsi_{u_{h}^k}(s_h^k, a_h^k)\|_{\bSigma_{1, k}^{-1}}\} - \EE_{s_h}\big[\min\{H, 2\beta \|\bpsi_{u_{h}^k}(s_h, \pi_h^k(s_h))\|_{\bSigma_{1, k}^{-1}}\}\big],
\end{align*}
then $x_{h,k}$ is obviously a martingale difference sequence bounded by $H$ w.r.t. $\{\cG_{h, k}\}_{h=1, k=1}^{H, K}$. Thus by Azuma-Hoeffding's inequality in Lemma~\ref{lm:azuma}, we have with probability at least $1 - \delta$, $\sum_{k=1}^K\sum_{h=1}^H x_h \le H\sqrt{2HK\log(1 / \delta)}$. Therefore,
\begin{align*}
    I_1 &= \sum_{k=1}^K \sum_{h=1}^H \min\{H, 2\beta \|\bpsi_{u_{h}^k}(s_h^k, a_h^k)\|_{\bSigma_{1, k}^{-1}}\} + \sum_{k=1}^K\sum_{h=1}^H x_h\\
    &\le 2\beta  \sum_{k=1}^K \sum_{h=1}^H \min\{1, \|\bpsi_{u_{h}^k}(s_h^k, a_h^k)\|_{\bSigma_{1, k}^{-1}}\} + H\sqrt{2HK\log(1 / \delta)}\\
    &\le 2\sqrt{2}\beta  \underbrace{\sum_{k=1}^K \sum_{h=1}^H \min\{1, \|\bpsi_{u_{h}^k}(s_h^k, a_h^k)\|_{\bSigma_{h, k}^{-1}}\}}_{I_2} + 4\beta H d\log(1 + KH^3 / \lambda) + H\sqrt{2HK\log(1 / \delta)}, 
\end{align*}
where the inequality on the second line is due to $2\beta \ge 2H\sqrt{d\log3} \ge H$ and the last inequality uses Lemma~\ref{lm:det} with $\bSigma_{1, k}^{-1} \succeq \bSigma_{h, k}^{-1}$ and $\det \bSigma_{1, k}^{-1} \le 2\det \bSigma_{1, k}^{-1}$ expect for $\tilde \cO(Hd)$ cases by Lemma~\ref{lm:batch}. By $\min\{1, \|\bpsi_{u_{h}^k}(s_h, \pi_h^k(s_h))\|_{\bSigma_{h, k}^{-1}}\} \le 1$ and $\|\bpsi_{u_{h}^k}(s_h^k, a_h^k)\|_2 \le H$  since $u_h^k \le H$, we can further bound the $\tilde \cO(Hd)$ terms where $\det \bSigma_{1, k}^{-1} > 2\det \bSigma_{1, k}^{-1}$. To bound $I_2$, by Lemma~\ref{lm:11}, using Cauchy-Schwarz inequality we have
\begin{align*}
    I_2 \le \sqrt{KH}\sqrt{\sum_{k=1}^K \sum_{h=1}^H \min\{1, \|\bpsi_{u_{h}^k}(s_h^k, a_h^k)\|^2_{\bSigma_{h, k}^{-1}}\}} \le \sqrt{2KHd\log(1 + KH^3/(d\lambda))},
\end{align*}

Plugging $I_2$ into $I_1$ then plugging $I_1$ into~\eqref{eq:total}. Let $\lambda = B^{-2}$, the summation of the value function $V_1^k(s_1^k)$ is bounded by
\begin{align*}
    \sum_{k=1}^K V_1^k(s_1^k) &\le 8\beta\Big(\sqrt{HKd\log(1 + KH^3B^2 / d)} + dH\log(1 + KH^3B^2)\Big) \\
    &\quad+ 2H\sqrt{2HK\log(1 / \delta)}.
\end{align*}
Taking $\delta = \delta / 3$, we can finalize the proof of Lemma~\ref{lm:sum-h}.
\end{proof}


\subsection{Proof of Lemma~\ref{lm:ci-b}}
The proof of this lemma is similar to the proof of Lemma~5.2 in~\citep{zhou2020nearly}. We extend their proof to a \emph{time varying} reward and \emph{homogeneous} setting, where the rewards (i.e., the exploration-driven reward function $r_h^k$) are different in different episode $k$. To prove this lemma, we need to introduce the Bernstein inequality for vector-valued martingales.
\begin{lemma}[Theorem~4.1,~\cite{zhou2020nearly}]\label{lm:zhou1}
Let $\{\cG_t\}_{t=1}^\infty$ be a filtration, $\{\xb_t, \eta_t\}_{t \ge 1}$ a stochastic process so that $\xb_t \in \RR^d$ is $\cG_t$-measurable and $\eta_t$ is $\cG_{t+1}$-measurable. Fix $R, L, \sigma, \lambda > 0, \bmu^* \in \RR^d$. For $t \ge 1$, let $y_t = \la \bmu^* , \xb_t\ra + \eta_t$. Suppose $\eta_t, \xb_t$ satisfy
\begin{align*}
    |\eta_t| \le R,\ \EE[\eta_t | \cG_t] = 0,\ \EE[\eta_t^2 | \cG_t] \le \sigma^2,\ \|\xb_t\|_2 \le L.
\end{align*}
Then for any $0 < \delta < 1$, with probability at least $1 - \delta$, we have 
\begin{align*}
    \forall t > 0,\ \bigg\|\sum_{\tau=1}^t \xb_\tau \eta_\tau\bigg\|_{\Ub_\tau^{-1}} \le \beta_t, \ \|\bmu_t - \bmu^*\|_{\Ub_t} \le \beta_t + \sqrt{\lambda}\|\bmu^*\|_2,
\end{align*}
where $\bmu_t = \Ub_t^{-1}\bbb_t, \Ub_t = \lambda \Ib + \sum_{\tau = 1}^t\xb_\tau\xb_\tau^\top, \bbb_t = \sum_{\tau = i}^t y_\tau\xb_\tau$, and 
\begin{align*}
    \beta_t = 8\sigma\sqrt{d\log(1 + tL^2 / d\lambda)\log(4t^2/\delta)} + 4R\log(4t^2/\delta).
\end{align*}
\end{lemma}

We also introduce the following lemma to analyze the error between the estimated variance $\bar \VV_h^k$ and the true variance $\VV_h^k$.
\begin{lemma}[Lemma~C.1,~\cite{zhou2020nearly}]\label{lm:zhou2}
Let $\VV_h^k(s, a)$ be as defined in~\eqref{eq:realv} and $\bar \VV_h^k(s, a)$ be as defined in~\eqref{eq:estvar}, then 
\begin{align*}
    |\VV_h^k(s, a) - \bar \VV_h^k(s, a)| &\le \min\Big\{H^2,  \|\bpsi_{[V_{h+1}^k]^2}(s, a)\|_{\tilde \bSigma_{1, k}^{-1}} \|\tilde \btheta_k - \btheta^*\|_{\tilde \bSigma_{1, k}}\Big\} \\
    &\quad+ \min\Big\{H^2,  2H\|\bpsi_{V_{h+1}^k}(s, a)\|_{\hat \bSigma_{1, k}^{-1}} \|\hat \btheta_k - \btheta^*\|_{\hat \bSigma_{1, k}}\Big\}.
\end{align*}
\end{lemma}
Equipped with these lemmas, we can start the proof of Lemma~\ref{lm:ci-b}.
\begin{proof}[Proof of Lemma~\ref{lm:ci-b}]
Recall the regression in~\eqref{eq:reg-bern}. For the regression on $\hat \bSigma, \hat \btheta$, let $\xb_h^k = \bpsi_{V_{h+1}^k}(s_h^k, a_h^k) / \bar \sigma_h^k$, and $\eta_h^k = V_{h+1}^k(s_{h+1}^k) / \bar \sigma_h^k - \la \btheta^*, \xb_h^k\ra$. Since $\bar \sigma_h^k \ge H/\sqrt{d}$ defined in~\eqref{eq:estv}, we get $\|\xb_h^k\|_2 \le \sqrt{d}, |\eta_h^k| \le \sqrt{d}$, thus one could verify that $\EE[[\eta^k_h]^2 | \cG_{h, k}] \le d$, $\EE[\eta^k_h | \cG_{h, k}] = 0$, from Lemma~\ref{lm:zhou1}, taking $t = (k-1)H$ we have
\begin{align*}
    \|\btheta^* - \hat \btheta_k\|_{\hat \bSigma_{1, k}} &\le 8d\sqrt{\log(1 + (k-1)H/\lambda)\log(4(k-1)^2H^2/\delta)}\\
    &\quad+ 4\sqrt{d}\log(4(k-1)^2H^2/\delta) + \sqrt{\lambda}B.
\end{align*}

For the regression of $\tilde \bSigma, \tilde \btheta$, $\xb_h^k = \bpsi_{[V_{h+1}^k]^2}(s_h^k, a_h^k)$ which directly implies $\|\xb_h^k\|_2 \le H^2$. Let $\eta_h^k = V_{h+1}^k(s_{h+1}^k)^2 - \la \btheta^*, \xb_h^k\ra$, one can easily verify that $|\eta_h^k| \le H^2$ and $\EE[\eta_h^k | \cG_{h, k}] = 0, \EE[[\eta_h^k]^2 | \cG_{h, k}] \le H^4$, thus using Lemma~\ref{lm:zhou1} again we have 
\begin{align*}
    \|\btheta^* - \tilde \btheta_k\|_{\tilde \bSigma_{1, k}} &\le 8H^2\sqrt{d\log(1 + (k-1)H/\lambda)\log(4(k-1)^2H^2/\delta)} \\
    &\quad+ 4H^2\log(4(k-1)^2H^2/\delta) + \sqrt{\lambda}B.
\end{align*}
Since $\lambda = B^{-2}$, if we select $\check \beta$ and $\tilde \beta$ as
\begin{align*}
    \check \beta &= 8d\sqrt{\log(1 + KHB^2/)\log(4K^2H^2/\delta)} + 4\sqrt{d}\log(4(k-1)^2H^2/\delta) + 1\\
    \tilde \beta &= 8H^2\sqrt{d\log(1 + KHB^2)\log(4K^2H^2/\delta)} + 4H^2\log(4K^2H^2/\delta) + 1,
\end{align*}
then with probability at least $1 - 2\delta$, for all $k \in [K + 1]$, $\|\btheta^* - \hat \btheta_k\|_{\hat \bSigma_{1, k}} \le \check \beta$, $\|\btheta^* - \tilde \btheta_k\|_{\tilde \bSigma_{1, k}} \le \tilde \beta$.

Next we are going to give the choice of $\hat \beta$ to make sure that $\|\btheta^* - \hat \btheta_k\|_{\hat \bSigma_{1, k}} \le \check \beta$ holds with high probability. The following proof is conditioned on that the aforementioned event $\|\btheta^* - \hat \btheta_k\|_{\hat \bSigma_{1, k}} \le \check \beta$, $\|\btheta^* - \tilde \btheta_k\|_{\tilde \bSigma_{1, k}} \le \tilde \beta$ holds, then from Lemma~\ref{lm:zhou2} we have 
\begin{align}
    &|\VV_h^k(s, a) - \bar \VV_h^k(s, a)|\notag  \\
    &\quad\le \min\Big\{H^2,  \tilde \beta\|\bpsi_{[V_{h+1}^k]^2}(s, a)\|_{\tilde \bSigma_{1, k}^{-1}} \Big\} + \min\Big\{H^2,  2\check \beta H\|\bpsi_{V_{h+1}^k}(s, a)\|_{\hat \bSigma_{1, k}^{-1}} \Big\}\notag  \\
    &\quad= E_h^k(s, a)\label{eq:estverr}
\end{align}
Again, let $\xb_h^k = \bpsi_{V_{h+1}^k}(s_h^k, a_h^k) / \bar \sigma_h^k$ to denote the context vector and $\eta_h^k = V_{h+1}^k(s_{h+1}^k) / \bar \sigma_h^k - \la \btheta^*, \xb_h^k\ra$ to denote the noise term, since $\|\btheta^* - \hat \btheta_k\|_{\hat \bSigma_{1, k}} \le \check \beta$, we have 
\begin{align*}
    \EE[[\eta_h^k]^2 | \cG_{h, k}] = \VV_h^k(s_h^k, a_h^k) / \nu_h^k \le (E_h^k(s_h^k, a_h^k) + \bar \VV_h^k(s_h^k, a_h^k)) / \nu_h^k \le 1,
\end{align*}
where the first inequality is from~\eqref{eq:estverr}, the second inequality holds because the definition of $\nu_h^k$ in~\eqref{eq:estv}. 

Therefore we have verified that the noise term $\eta_h^k$ is a zero-mean random variable conditioned on $\cG_{h, k}$ and $\EE[[\eta_h^k]^2 | \cG_{h, k}] \le 1$. In that case, using Lemma~\ref{lm:zhou1} again we could get with probability at least $1 - \delta$,
\begin{align}
    \|\btheta^* - \hat \btheta_k\|_{\hat \bSigma_{1, k}} &\le 8\sqrt{d(1 + (k - 1)H/\lambda)\log(4(k-1)^2H^2/\delta)} \\
    &\quad+ 4\sqrt{d}\log(4(k-1)^2H^2/\delta) + \sqrt{\lambda}B,\label{eq:final}
\end{align}
again, since $\lambda = B^{-2}$, if we select $\hat \beta$ as 
\begin{align*}
    \hat \beta = 8\sqrt{d(1 + KHB^2)\log(4K^2H^2/\delta)} + 4\sqrt{d}\log(4K^2H^2/\delta) + 1,
\end{align*}
then $\|\btheta^* - \hat \btheta_k\|_{\hat \bSigma_{1, k}} \le \hat \beta$ with probability at least $1 - \delta$ for all $k \in [K + 1]$.

Next, for the regression of $\btheta_{K + 1}, \bSigma_{1, K + 1}$, by Lemma~\ref{lm:ci-h}, we obtain the same result with the selection of $\beta$ as
\begin{align*}
    \beta = H\sqrt{d\log\bigg(\frac{1 + KH^3/\lambda}{\delta}\bigg)} + B\sqrt{\lambda},
\end{align*}
which suggests that with probability at least $1 - \delta$, $\|\btheta_{K + 1} - \btheta^*\|_{\bSigma_{1, K + 1}} \le \beta$. Then taking union bound with all aforementioned event 
$\|\btheta^* - \hat \btheta_k\|_{\hat \bSigma_{1, k}} \le \check \beta,\ \|\btheta^* - \tilde \btheta_k\|_{\tilde \bSigma_{1, k}} \le \tilde \beta$, $\|\btheta^* - \hat \btheta_k\|_{\hat \bSigma_{1, k}} \le \hat \beta$, we have all these events mentioned in this proof holds with probability at least $1 - 4\delta$. Replace $\delta$ with $\delta / 12$, we obtain our final results.

Next, for the regression of $\btheta_{K + 1}, \bSigma_{1, K + 1}$, by Lemma~\ref{lm:ci-h}, we obtain the same result with the selection of $\beta$ as
\begin{align*}
    \beta = H\sqrt{d\log\bigg(\frac{1 + KH^3/\lambda}{\delta}\bigg)} + B\sqrt{\lambda},
\end{align*}
which suggests that with probability at least $1 - \delta$, $\|\btheta_{K + 1} - \btheta^*\|_{\bSigma_{1, K + 1}} \le \beta$. Again, taking an additional union bound, with probability at least $1-4\delta$, all events mentioned in this proof hold. Replace $\delta$ with $\delta / 12$, we obtain our final results. 
\end{proof}

\subsection{Proof of Lemma~\ref{lm:sum-b}}
The proof of this lemma borrows some intuition from the proof of Theorem~5.3 in~\citep{zhou2020nearly}. Unlike \citet{zhou2020nearly} that deals the fixed reward and time-inhomogeneous setting, we need to extend their proof in order to deal with the time-varying reward and time-homogeneous setting.

The next lemmas shows the relationship between the summation of $\nu_h^k$ and the difference between $V_h^k(s)$ calculated in Algorithm~\ref{alg:explore-bern} and $V_h^{\pi^k}(s; \{r_h^k\}_{h=1, k=1}^{H, K})$

\begin{lemma}\label{lm:ieq1}
Let $V_h^k, \nu_h^k$ be defined in Algorithm~\ref{alg:explore-bern}. Then if the condition in Lemma~\ref{lm:ci-b} holds, the following inequality holds with probability at least $1 - 2\delta$,
\begin{align*}
    \sum_{k=1}^K[V_1^k(s_1^k) - V_1^{\pi^k}(s_1^k)] &\le 4\sqrt{d}\hat \beta\sqrt{\sum_{k=1}^K\sum_{h=1}^H\nu_h^k}\sqrt{\log(1 + KHB^2)} \\
    &\quad+ 2H^2d\log(1 + KHB^2d) + H\sqrt{2KH\log(1 / \delta)}\\
    \sum_{k=1}^K \sum_{h=1}^H [\PP(V_{h+1}^k - V_{h+1}^{\pi^k})](s_h^k, a_h^k)  &\le 4\sqrt{d}H\hat \beta\sqrt{\sum_{k=1}^K\sum_{h=1}^H\nu_h^k}\sqrt{\log(1 + KHB^2)}\\
    &\quad + 2H^3d\log(1 + KHB^2d) + 2H^2\sqrt{2KH\log(1 / \delta)},
\end{align*}
\end{lemma}

\begin{lemma}\label{lm:ieq2}
Let $V_h^k$, $\nu_h^k$ be defined in Algorithm~\ref{alg:explore-bern}. Then if the condition in Lemma~\ref{lm:ci-b} holds, with probability at least $1 - \delta$,
\begin{align*}
    \sum_{k=1}^K\sum_{h=1}^H \nu_h^k &\le \frac{H^3K}{d} + 3H^2K + 3H^3\log(1 / \delta) +  2H\sum_{k=1}^K\sum_{h=1}^H [\PP(V_{h+1}^k - V_{h+1}^{\pi^k}](s_h^k, a_h^k) \\
    &\quad +2 \tilde \beta\sqrt{KHd\log(1 + KH^5B^2/d)} + 4 \tilde \beta Hd\log(1 + KH^5B^2/d) \\
    &\quad+ 8H^2\check \beta\sqrt{KHd\log(1 + KHB^2)} + 8H^3d\check \beta\log(1 + KHdB^2).
\end{align*}
\end{lemma}

Equipped with these two lemmas, we can start to prove Lemma~\ref{lm:sum-b}.
\begin{proof}[Proof of Lemma~\ref{lm:sum-b}]
In this proof, we use $\tilde \cO(\cdot)$ to ignore all constant and log terms to simplify the results. Recall the selection of $\beta, \hat \beta, \check \beta, \tilde \beta$, we have $\beta = \tilde \cO(H\sqrt{d})$, $\hat \beta = \tilde \cO(\sqrt{d})$, $\check \beta = \tilde \cO(d)$, $\tilde \beta = \tilde \cO(H^2\sqrt{d})$. Therefore  Lemma~\ref{lm:ieq1} could be simplified as
\begin{align}
    \sum_{k=1}^K \sum_{h=1}^H [\PP(V_{h+1}^k - V_{h+1}^{\pi^k})](s_h^k, a_h^k)  &\le \tilde \cO\Bigg(Hd\sqrt{\sum_{k=1}^K\sum_{h=1}^H\nu_h^k} + H^3d + \sqrt{KH^5}\Bigg).\label{eq:z1}
\end{align}
Lemma~\ref{lm:ieq2} could also be simplified as
\begin{align}
    \sum_{k=1}^K\sum_{h=1}^H \nu_h^k &\le \tilde \cO\Bigg(\frac{H^3K}{d} + H^2K +  H\sum_{k=1}^K\sum_{h=1}^H [\PP(V_{h+1}^k - V_{h+1}^{\pi^k})](s_h^k, a_h^k) + \sqrt{KH^5d^3} + H^3d^2\Bigg).\label{eq:z2}
\end{align}
Let $\sqrt{\sum_{k=1}^K\sum_{h=1}^H \nu_h^k} = x$, plugging~\eqref{eq:z1} into~\eqref{eq:z2}, we have
\begin{align*}
    x^2 \le \tilde \cO(H^3Kd^{-1} + H^2K +  H^2dx + H^4d + \sqrt{KH^7} + \sqrt{KH^5d^3} + H^3d^2),
\end{align*}
Since the quadratic inequality $x^2 \leq \tilde \cO(bx+c)$ indicates that $x \leq \cO(b+\sqrt{c})$, setting
\begin{align*}
    b = \tilde \cO(H^2d), c = \tilde \cO(H^3Kd^{-1} + H^2K + H^4d + \sqrt{KH^7} + \sqrt{KH^5d^3} + H^3d^2),
\end{align*}
hence
\begin{align}
    \sqrt{\sum_{k=1}^K\sum_{h=1}^H \nu_h^k} &\le \tilde \cO(H^2d + \sqrt{H^3K/d} + H\sqrt{K} + H^2\sqrt{d} + d\sqrt{H^3} + (KH^7)^{1/4} + (KH^5d^3)^{1/4}) \\
    &= \tilde \cO(\sqrt{H^3K / d} + H\sqrt{K}) + o(\sqrt{K}).\label{eq:xx1}
\end{align}
Plugging \eqref{eq:xx1} back to Lemma~\ref{lm:ieq1}, we have 
\begin{align}
    \sum_{k=1}^K[V_1^k(s_1^k) - V_1^{\pi^k}(s_1^k)] \le \tilde \cO(\sqrt{H^3Kd} + Hd\sqrt{K}) + o(\sqrt{K}).\label{eq:b2}
\end{align}
Next we are going to show the bound of the summation over $V_1^{\pi^k}(s_1^k)$, note that this value function is bounded by $H$ and from Bellman equality, we have 
\begin{align*}
    V_h^{\pi^k}(s_1^k) = r_h^k(s_1^k, a_1^k) + [\PP V_{h+1}^{\pi^k}](s_h^k, a_h^k),
\end{align*}
taking summation over $h \in [H], k \in [K]$ then
\begin{align*}
    \sum_{k=1}^KV_1^{\pi^k}(s_1^k) &=  \sum_{k=1}^K\sum_{h=1}^H r_h^k(s_1^k, a_1^k) +  \sum_{k=1}^K\sum_{h=1}^H [\PP V_{h+1}^{\pi^k}](s_h^k, a_h^k) - V_{h+1}^{\pi^k}(s_{h+1}^k)\\
    &\le \sum_{k=1}^K\sum_{h=1}^H\min\{1, 2\beta \|\bpsi_{u_h^k}(s_h^k, a_h^k)\|_{\bSigma_{1, k}^{-1}} / H\} + H\sqrt{HK\log(1 / \delta)},
\end{align*}
where the last inequality holds due to Azuma-Hoeffding's inequality i.e. Lemma~\ref{lm:azuma}. For the first term,
\begin{align*}
    \sum_{k=1}^K\sum_{h=1}^H\min\{1, 2\beta \|\bpsi_{u_h^k}(s_h^k, a_h^k)\|_{\bSigma_{1, k}^{-1}} / H\} \le \frac{2\beta}{H}\underbrace{\sum_{k=1}^K\sum_{h=1}^H \min\{1, \|\bpsi_{u_h^k}(s_h^k, a_h^k)\|_{\bSigma_{1, k}^{-1}}\}}_{I_1},
\end{align*}
where the inequality is due to $\beta \ge H\sqrt{\log(12)} \ge H / 2$. Using Lemma~\ref{lm:batch} and Lemma~\ref{lm:det} with $\bSigma_{1, k}^{-1} \succeq \bSigma_{h, k}^{-1}$ and $\det \bSigma_{1, k}^{-1} \le 2\det \bSigma_{h, k}^{-1}$ except for $\tilde \cO(Hd)$ steps mentioned in Lemma~\ref{lm:batch}, setting $\lambda = B^{-2}$, we have
\begin{align*}
    I_1 &\le 2Hd\log(1 + KH^3B^2) + \sqrt{2} \sum_{k=1}^K\sum_{h=1}^H \min\{1, \|\bpsi_{u_h^k}(s_h^k, a_h^k)\|_{\bSigma_{h, k}^{-1}}\} \\
    &\le 2Hd\log(1 + KH^3B^2) + \sqrt{2HK} \sqrt{\sum_{k=1}^K\sum_{h=1}^H \min\{1, \|\bpsi_{u_h^k}(s_h^k, a_h^k)\|^2_{\bSigma_{h, k}^{-1}}\}}\\
    &\le 2Hd\log(1 + KH^3B^2) + 2\sqrt{HKd\log(1 + KH^3B^2/d)}.
\end{align*}
Therefore, since $\beta = \tilde \cO(H\sqrt{d})$, then
\begin{align}
    \sum_{k=1}^KV_1^{\pi^k}(s_1^k) &\le 4\beta d\log(1 + KH^3B^2) + 4\beta \sqrt{Kd\log(1 + KH^3B^2 / d)/H} + \sqrt{H^3K\log(1 / \delta)}\\
    &\le \tilde \cO(d\sqrt{KH} + \sqrt{KH^3}) + o(\sqrt{K}).\label{eq:b1}
\end{align}
Adding~\eqref{eq:b2} and~\eqref{eq:b1} together, we have the following result,
\begin{align*}
    \sum_{k=1}^K V_1^k(s_1^k) \le \tilde \cO(\sqrt{H^3Kd} + Hd\sqrt{K}) + o(\sqrt{K}).
\end{align*}
By taking the union bound, this inequality holds with probability at least $1 - 4\delta$. Since $\delta$ only appears in the logarithmic terms, thus changing $\delta$ to $\delta / 12$ will not affect the result.
\end{proof}
\section{Proof of Auxiliary Lemmas in Appendix~\ref{sec:A}}
\subsection{Proof of Lemma~\ref{lm:batch}}
\begin{proof}[Proof of Lemma~\ref{lm:batch}]
We want to know how many pairs of $(h, k)$ exists such that $\det \Ub_{h, k} \ge 2\det \Ub_{1, k}$.

Furthermore, we have if there exists $k \in [K]$ such that $\det \Ub_{1, k + 1} \le 2\det \Ub_{1, k}$,
then it is obvious that for all $h \in [H]$, we have $\det \Ub_{h, k} \le \det \Ub_{1, k + 1} \le 2\det\Ub_{1, k}$.

Therefore, suppose there exists a set $\cK \subset [K]$ such that for all $k \notin \cK$, $\det \Ub_{1, k + 1} \le 2\det \Ub_{1, k}$ and for all $k \in \cK$, $\det \Ub_{1, k + 1} > 2\det \Ub_{1, k}$, then the pair of $(h, k)$ such that $\det \Ub_{h, k} \ge 2\det \Ub_{1, k}$ is upper bounded by $H|\cK|$.

Notice that for all $k \in \cK$, $\det \Ub_{1, k + 1} > 2\det \Ub_{1, k}$, it is easy to show that 
\begin{align*}
    \det \Ub_{1, K + 1} > 2^{|\cK|}\det \Ub_{1, 1} =  2^{|\cK|}\lambda^d,
\end{align*}
where the last inequality comes from $\Ub_{1, 1} = \lambda \Ib \in \RR^{d \times d}$. Notice that $\det \Ub \le \|\Ub\|_{2}^d$, taking log we have
\begin{align}
    d\log(\|\Ub_{1, K + 1}\|_{2}) \ge \log \det \Ub_{1, K + 1} > |\cK|\log 2 + d \log \lambda.\label{eq:q2}
\end{align}
From the definition of $\Ub_{1, K + 1}$, by triangle inequality,
\begin{align}
   \|\Ub_{1, K + 1}\|_{2} \le \lambda + \sum_{k=1}{K}\sum_{h=1}^H \|\xb_h^k\xb_h^{k\top}\|_2 \le \lambda + KH\|\xb_h^k\|_2^2 \le \lambda + KHL^2,\label{eq:q1}
\end{align}
where the last inequality is due to $\|\xb\|_2 \le L$ from the statement of the lemma. Therefore we conclude our proof by merging~\eqref{eq:q2} and~\eqref{eq:q1} together to get
\begin{align*}
    |\cK|\log 2 < d\log(1 + HKL^2 / \lambda),
\end{align*}
noticing $\log 2 \ge 1 / 2$ we can get the result claimed in the lemma.
\end{proof}

\subsection{Proof of Lemma~\ref{lm:ieq1}}
\begin{proof}[Proof of Lemma~\ref{lm:ieq1}]
Assume that the condition in Lemma~\ref{lm:ci-b} holds, then
\begin{align*}
    &V_h^k(s_h^k) - V_h^{\pi^k}(s_h^k) \\
    &\quad\le \la \hat \btheta_k, \bpsi_{V_{h+1}^k}(s_h^k, a_h^k)\ra - [\PP V_{h+1}^{\pi^k}](s_h^k, a_h^k) + \hat \beta \|\bpsi_{V_{h+1}^k}(s_h^k, a_h^k)\|_{\hat \bSigma_{1, k}^{-1}}\\
    &\quad\le \|\hat \btheta_k - \btheta^*\|_{\hat \bSigma_{1, k}} \|\bpsi_{V_{h+1}^k}(s_h^k, a_h^k)\|_{\hat \bSigma_{1, k}^{-1}} + [\PP V_{h+1}^{k} - V_{h+1}^{\pi^k}](s_h^k, a_h^k) + \hat \beta \|\bpsi_{V_{h+1}^k}(s_h^k, a_h^k)\|_{\hat \bSigma_{1, k}^{-1}}\\
    &\quad\le 2\hat \beta \|\bpsi_{V_{h+1}^k}(s_h^k, a_h^k)\|_{\hat \bSigma_{1, k}^{-1}} + [\PP V_{h+1}^{k} - V_{h+1}^{\pi^k}](s_h^k, a_h^k),
\end{align*}
where the first inequality holds due to the definition of $V_h^k$, the second inequality holds due to Cauchy-Schwarz inequality and the third one holds due to the condition~\eqref{eq:bern-cond} in Lemma~\ref{lm:ci-b}. Notice that $V_h^k - V_h^{\pi^k} \le H$, we have
\begin{align*}
    V_h^k(s_h^k) - V_h^{\pi^k}(s_h^k) &\le \min\{H, 2\hat \beta \|\bpsi_{V_{h+1}^k}(s_h^k, a_h^k)\|_{\hat \bSigma_{1, k}^{-1}}\} + [\PP V_{h+1}^{k} - V_{h+1}^{\pi^k}](s_h^k, a_h^k)
\end{align*}
Taking summation over $k \in [K]$ and $h \in [H]$, we have
\begin{align}
    \sum_{k=1}^K [V_1^k(s_1^k) - V_1^{\pi_k}(s_1^k)] &\le \sum_{k=1}^K\sum_{h=1}^H \min\{H, 2\hat \beta \|\bpsi_{V_{h+1}^k}(s_h^k, a_h^k)\|_{\hat \bSigma_{1, k}^{-1}}\}\notag \\
    &\quad + \sum_{k=1}^K\sum_{h=1}^H\Big[[\PP V_{h+1}^k - V_{h+1}^{\pi^k}](s_h^k, a_h^k) - [V_{h+1}^k(s_{h+1}^k) - V_{h+1}^{\pi^k}(s_{h+1}^k)] \Big]\notag \\
    &\le \underbrace{\sum_{k=1}^K\sum_{h=1}^H\min\{H, 2\hat \beta \|\bpsi_{V_{h+1}^k}(s_h^k, a_h^k)\|_{\hat \bSigma_{1, k}^{-1}}\}}_{I_1} + H\sqrt{2KH\log(1 / \delta)},\label{eq:r1}
\end{align}
where the second inequality is a direct result of Azuma-Hoeffding's inequality as in Lemma~\ref{lm:azuma}.

Next we bound $I_1$. Recall the update rule of $\hat \bSigma_{h, k}$, notice that $\bar \sigma_h^k \ge H/\sqrt{d}$ and $\|\bpsi_{V_{h+1}^k}(s_h^K, a_h^K)\|_2 \le H$ from $V_{h+1}^k \le H$, it is easy to verify that $\|\bpsi_{V_{h+1}^k}(s_h^K, a_h^K) / \hat \sigma_h^k\|_2 \le \sqrt{d}$. 
Hence
\begin{align*}
    I_1 &\le \sqrt{2} \sum_{k=1}^K\sum_{h=1}^H\min\{H, 2\hat \beta \|\bpsi_{V_{h+1}^k}(s_h^k, a_h^k)\|_{\hat \bSigma_{h, k}^{-1}}\} + 2H^2d\log(1 + KHd/\lambda)\\
    &\le \sqrt{2}\max\{\sqrt{d}, 2\hat \beta\}\sum_{k=1}^K\sum_{h=1}^H\bar \sigma_h^k \min\{1, \|\bpsi_{V_{h+1}^k}(s_h^k, a_h^k) / \bar \sigma_h^k\|_{\hat \bSigma_{h, k}^{-1}}\} + 2H^2d\log(1 + KHd/\lambda)\\
    &\le 2\sqrt{2}\hat \beta\sqrt{\sum_{k=1}^K\sum_{h=1}^H\nu_h^k}\sqrt{\sum_{k=1}^K\sum_{h=1}^H\min\{1, \|\bpsi_{V_{h+1}^k}(s_h^k, a_h^k) / \bar \sigma_h^k\|_{\hat \bSigma_{h, k}^{-1}}^2\}} + 2H^2d\log(1 + KHd/\lambda)\\
    &\le 4\hat \beta \sqrt{d} \sqrt{\sum_{k=1}^K\sum_{h=1}^H\nu_h^k}\sqrt{\log (1 + KH/\lambda)} + 2H^2d\log(1 + KHd/\lambda),
\end{align*}
where the first inequality, similar to the corresponding proof in Lemma~\ref{lm:sum-h}, is a direct implication of Lemma~\ref{lm:batch} and Lemma~\ref{lm:det} with $\hat \bSigma_{1, k}^{-1} \succeq \hat \bSigma_{h, k}^{-1}$ and $\det \bSigma_{1, k}^{-1} \le 2\det \hat \bSigma_{1, k}^{-1}$ except for $\tilde \cO(Hd)$ cases mentioned in Lemma~\ref{lm:batch}, the second inequality moves $\bar \sigma_h^k$ outside, the third inequality holds because $\hat \beta \ge 4\sqrt{d}\log 12 \ge \sqrt{d}$ and Cauchy-Schwarz inequality, and the forth inequality holds due to Lemma~\ref{lm:11}.
Plugging $I_1$ into~\eqref{eq:r1} and let $h' = 1, \lambda = B^{-2}$, we have
\begin{align*}
    \sum_{k=1}^K[V_1^k(s_1^k) - V_1^{\pi^k}(s_1^k)] &\le 4\sqrt{d}\hat \beta\sqrt{\sum_{k=1}^K\sum_{h=1}^H\nu_h^k}\sqrt{\log(1 + KHB^2)} \\
    &\quad+ 2H^2d\log(1 + KHB^2d) + H\sqrt{2KH\log(1 / \delta)}.
\end{align*}
Furthermore, by Azuma-Hoeffding's inequality as in Lemma~\ref{lm:azuma},
\begin{align*}
    \sum_{k=1}^K \sum_{h=1}^H \PP[V_{h+1}^k - V_{h+1}^{\pi^k}](s_h^k, a_h^k) &= \sum_{k=1}^K\sum_{h=2}^H [V_h^k - V_h^{\pi^k}](s_h^k)\\
    &\quad + \sum_{k=1}^K\sum_{h=1}^H \Big[\PP(V_{h+1}^k - V_{h+1}^{\pi^k}](s_h^k, a_h^k) - [V_{h+1}^k - V_{h+1}^{\pi^k})](s_{h+1}^k)\\
    &\le 4\sqrt{d}H\hat \beta\sqrt{\sum_{k=1}^K\sum_{h=1}^H\nu_h^k}\sqrt{\log(1 + KHB^2)}\\
    &\quad + 2H^3d\log(1 + KHB^2d) + (H + 1)H\sqrt{2KH\log(1 / \delta)},
\end{align*}
which becomes the second part of the statement in the lemma. Using $H + 1 \le 2H$ we can get the result claimed in the lemma.
\end{proof}

\subsection{Proof of Lemma~\ref{lm:ieq2}}
To begin with, we will first show the total variance lemma originally introduced in~\citep{jin2018q}. 
\begin{lemma}[Total variance lemma, Lemma C.5, \cite{jin2018q}]\label{lm:totalvar}
\footnote{The original Lemma C.5 in \citet{jin2018q} holds for the identical reward functions, i.e., $r_h^1 = \cdots =  r_h^K$. Their lemma also holds for the general case $r_h^1 \neq \cdots \neq r_h^K$ without changing their proof. } With probability at least $1 - \delta$, we have 
\begin{align*}
    \sum_{k=1}^K\sum_{h=1}^H[\VV V^{\pi^k}_h(\cdot; \{r_h^k\}_{h=1}^H)](s, a) \le 3H^2K + 3H^3\log(1 / \delta).
\end{align*}
\end{lemma}

\begin{proof}[Proof of Lemma~\ref{lm:ieq2}]
Assume the condition in Lemma~\ref{lm:ci-b} holds, we have with probability at least $1 - \delta$,
\begin{align}
    \sum_{k=1}^K\sum_{h=1}^H \nu_h^k &\le \sum_{k=1}^K\sum_{h=1}^H \bigg(\frac{H^2}{d} + \bar \VV_h^k(s_h^k, a_h^k) + E_h^k(s_h^k, a_h^k)\bigg) \notag \\
    &= \frac{H^3K}{d} + \underbrace{\sum_{k=1}^K\sum_{h=1}^H\Big( [\VV_h V_{h+1}^k](s_h^k, a_h^k) - [\VV_h V_{h+1}^{\pi^k}](s_h^k, a_h^k)\Big)}_{I_1} + 2\underbrace{\sum_{k=1}^H\sum_{h=1}^H E_h^k(s_h^k, a_h^k)}_{I_2}\notag \\
    &\quad + \underbrace{\sum_{k=1}^K\sum_{h=1}^H [\VV_h V_{h+1}^{\pi^k}](s_h^k, a_h^k)}_{I_3} + \underbrace{\sum_{k=1}^K\sum_{h=1}^H\Big[\bar \VV_h^k(s_h^k, a_h^k) - [\VV_h V_{h+1}^k](s_h^k, a_h^k) - E_h^k\Big]}_{I_4}\notag \\
    &\le \frac{H^3K}{d} + I_1 + I_2 + 3H^2K + 3H^3\log(1 / \delta),\label{eq:insert}
\end{align}
where the value function $V_{h}^{\pi^k}(s)$ is short for $V_h^{\pi^k}(s; \{r_h^k\}_{h=1}^H)$ for simplicity. The first inequality is from the definition of $\nu_h^k$ in~\eqref{eq:estv}, while the last inequality is from Lemma~\ref{lm:totalvar} to control $I_3$. $I_4 \le 0$ is due to Lemma~\ref{lm:ci-b}. Next we are about to bound $I_1$ and $I_2$ separately. 

Since the estimated value function $V_{h+1}^k$ and the real value function $V_{h+1}^{\pi^k}$ are both bounded by $[0, H]$, we have 
\begin{align*}
    I_1 \le \sum_{k=1}^K\sum_{h=1}^H \big[\PP ([V_{h+1}^k]^2 - [V_{h+1}^{\pi^k}]^2)\big](s_h^k, a_h^k) \le 2H\sum_{k=1}^K\sum_{h=1}^H [\PP(V_{h+1}^k - V_{h+1}^{\pi^k})](s_h^k, a_h^k).
\end{align*}
For term $I_2$, we have
\begin{align*}
    I_2 &\le \sum_{k=1}^K\sum_{h=1}^H \min\{H^2, \tilde \beta \|\bpsi_{[V_{h+1}^k]^2}(s_h^k, a_h^k)\|_{\tilde \bSigma_{1, k}^{-1}}\} + \sum_{k=1}^K\sum_{h=1}^H \min\{H^2, 2H\check \beta \|\bpsi_{V_{h+1}^k}(s, a)\|_{\hat \bSigma_{1, k}^{-1}}\}\\
    &\le \max\{H^2, \tilde \beta\} \sum_{k=1}^K\sum_{h=1}^H \min\{1,  \|\bpsi_{[V_{h+1}^k]^2}(s_h^k, a_h^k)\|_{\tilde \bSigma_{1, k}^{-1}}\}\\
    &\quad+ \sum_{k=1}^K\sum_{h=1}^H\max\{H^2, 2H\check \beta\bar\sigma_h^k\} \min\big\{1, \big\|\bpsi_{V_{h+1}^k}(s, a) / \bar \sigma_h^k\big\|_{\hat \bSigma_{1, k}^{-1}}\big\}.
\end{align*}
Noticing that from the definition of $\nu_h^k$, 
\begin{align*}
    \nu_k^h = \max\{H^2 / d, \bar \VV_h^k(s_h^k, a_h^k) + E_h^k(s_h^k, a_h^k)\} \le \max\{H^2 / d, H^2 + 2H^2\} = 3H^2,
\end{align*}
thus $\bar \sigma_h^k = \sqrt{\nu_h^k} \le 2H$. Recall that $\tilde \beta \ge 4H^2\log(12) \ge H^2$ and $\check \beta \ge 1$, we have
\begin{align*}
    I_2 &\le \tilde \beta \underbrace{\sum_{k=1}^K\sum_{h=1}^H \min\{1,  \|\bpsi_{[V_{h+1}^k]^2}(s_h^k, a_h^k)\|_{\tilde \bSigma_{1, k}^{-1}}\}}_{I_5} + 4H^2\check \beta\underbrace{\sum_{k=1}^K\sum_{h=1}^H\min\big\{1, \big\|\bpsi_{V_{h+1}^k}(s, a) / \bar \sigma_h^k\big\|_{\hat \bSigma_{1, k}^{-1}}\big\}}_{I_6}.
\end{align*}
For $I_5$, using Lemmas~\ref{lm:batch} and~\ref{lm:det} with $\tilde \bSigma_{1, k}^{-1} \succeq \tilde \bSigma_{h, k}^{-1}$ and $\det \tilde \bSigma_{1, k}^{-1} \le 2\det \tilde \bSigma_{1, k}^{-1}$ except for $\tilde \cO(Hd)$ cases mentioned in Lemma~\ref{lm:batch}, we have
\begin{align*}
    I_5 &\le \sqrt{2} \sum_{k=1}^K\sum_{h=1}^H \min\{1,  \|\bpsi_{[V_{h+1}^k]^2}(s_h^k, a_h^k)\|_{\tilde \bSigma_{h, k}^{-1}}\} + 2Hd\log(1 + KH^5/d\lambda)\\
    &\le \sqrt{2KH}\sqrt{\sum_{k=1}^K\sum_{h=1}^H\min\{1,  \|\bpsi_{[V_{h+1}^k]^2}(s_h^k, a_h^k)\|^2_{\tilde \bSigma_{h, k}^{-1}}\}} + 2Hd\log(1 + KH^5/d\lambda)\\
    &\le 2\sqrt{KHd\log(1 + KH^5/d\lambda)} + 2Hd\log(1 + KH^5/d\lambda),
\end{align*}
where the first inequality is a direct implication from Lemma~\ref{lm:batch} and the second inequality is due to Cauchy-Schwarz inequality. The third inequality utilizes Lemma~\ref{lm:11}. As for $I_6$, we have
\begin{align*}
    I_6 &\le \sqrt{2}\sum_{k=1}^K\sum_{h=1}^H\min\big\{1, \big\|\bpsi_{V_{h+1}^k}(s, a) / \bar \sigma_h^k\big\|_{\hat \bSigma_{h, k}^{-1}}\big\} + 2Hd\log(1 + KHd / \lambda)\\
    &\le \sqrt{2KH}\sqrt{\sum_{k=1}^K\sum_{h=1}^H\min\big\{1, \big\|\bpsi_{V_{h+1}^k}(s, a) / \bar \sigma_h^k\big\|_{\hat \bSigma_{h, k}^{-1}}\big\}} + 2Hd\log(1 + KHd / \lambda)\\
    &\le 2\sqrt{KHd\log(1 + KH/\lambda)} + 2Hd\log(1 + KHd / \lambda).
\end{align*}

Finally, plugging $I_5, I_6$ into $I_2$ and $I_1, I_2$ into~\eqref{eq:insert} we have
\begin{align*}
    \sum_{k=1}^K\sum_{h=1}^H \nu_h^k &\le \frac{H^3K}{d} + 3H^2K + 3H^3\log(1 / \delta) +  2H\sum_{k=1}^K\sum_{h=1}^H [\PP (V_{h+1}^k - V_{h+1}^{\pi^k}](s_h^k, a_h^k) \\
    &\quad +2 \tilde \beta\sqrt{KHd\log(1 + KH^5/d\lambda)} + 4 \tilde \beta Hd\log(1 + KH^5/d\lambda) \\
    &\quad+ 8H^2\check \beta\sqrt{KHd\log(1 + KH/\lambda)} + 8H^3d\check \beta\log(1 + KHd / \lambda).
\end{align*}
Let $\lambda = B^{-2}$ we could get the result in the statement of the lemma.
\end{proof}

\section{Proof of Lower Bound}\label{sec:proofL}

In this section, we will give the detailed proof of the sample complexity lower bound. We start with verifying that the MDP structure as shown in Figure~\ref{fig:hardmdp} is a linear mixture MDP satisfying Definition~\ref{def:mdp}.
\subsection{Verification of the MDP structure}

We will first show that the $\ell_2$ norm of $\btheta$ is controlled. Recall the $\btheta_i$ is set by $\bTheta = \Big\{\btheta_i | \btheta_i = \begin{pmatrix} \sqrt{2}, \alpha\tilde \btheta_i^\top / \sqrt{d}\end{pmatrix}^\top\Big\}$ where $\tilde \btheta_i = \xb_i \in \cM$, we can have that $\|\btheta_i\|_2 = \sqrt{2 + \alpha^2}$, therefore, as long as the parameter $\alpha$ is an absolute constant, the $\ell_2$ norm of $\btheta$ is controlled. Next, considering a function $V \le D$, we have 
\begin{align*}
    \|\bpsi_V(S_1, a_i)\|_2 \quad= \bigg\|\begin{pmatrix}\frac{V(S_{2,1}) + V(S_{2,2})}{2\sqrt{2}} & (V(S_{2,1}) - V(S_{2,2})) \frac{\ab_i^\top}{\sqrt{2d}}
    \end{pmatrix}\bigg\|_2 \le \sqrt{\frac{D^2}2 + \frac{D^2d}{2d}} = D,
\end{align*}
which shows that the MDP structure satisfies Definition~\ref{def:mdp}.

\subsection{Proof of Theorem~\ref{thm:mainL}}
 We denote the $d-1$-dimension binary vector set as $\cB_{d-1} = \{\xb | \xb \in \RR^{d-1}, [\xb]_i \in \{-1, 1\}\}$. The next lemma shows that the binary set $\cM$ exists. 

\begin{lemma}\label{lm:packing}
Given $\gamma \in (0, 1)$, there exists a $\cM \subset \cB_{d - 1}$ such that for any two different vector $\xb, \xb' \in \cM, \la \xb, \xb'\ra \le (d - 1)\gamma$, and the log-cardinality of the proposed set is bounded as $\log(|\cM|) < (d-1)\gamma^2/4$.
\end{lemma}

With this lemma, we can construct a set $\cM$ with $|\cM| = \lceil \exp(d\gamma^2/4)\rceil - 1$ where $\gamma = \frac12$. It is easy to verify that $|\cM| < \exp(d\gamma^2/4)$ and 
\begin{align}
    \log(|\cM|) &\ge \log(\exp(d\gamma^2/4) - 1) \notag \\
    &\ge \frac {(d-1)\gamma^2}{4} + \log(1 - \exp(-(d-1)\gamma^2/4)) \notag \\
    &\ge  \frac{(d-1)\gamma^2}{4} - 3,\label{eq:size}
\end{align}
where the last inequality holds since $\gamma = \frac12$ and $d \ge 2$, we have $\log(1 - \exp(-(d-1)\gamma^2/4) \ge -3$. From Lemma~\ref{lm:packing}, we know that for any two different vectors $\xb, \xb' \in \cM, \la \xb, \xb'\ra \le (d - 1)/2$.

Next lemma establishes the lower bound of sample complexity for any algorithm to estimate the true parameter $\btheta \in \bTheta$ of the proposed linear mixture MDP, from the sampled state-action pairs of this MDP.

\begin{lemma}\label{lm:class}
Suppose an algorithm estimates the underlying parameter $\btheta$ by building an estimator $\hat\btheta$ from $K$ sampled trajectories. If the algorithm guarantees that $\EE_{\btheta \sim \mathrm{Unif}\bTheta} [\ind(\hat \btheta = \btheta)] \ge 1 - \delta$, we have
\begin{align*}
    \delta \ge 1 - \bigg(\frac{d-1}{16} - 3\bigg)^{-1}\bigg(\log 2 + \frac{4K\alpha^2}{2 - \alpha^2}\bigg).
\end{align*}
\end{lemma}

Finally, the next lemma suggests that if $\alpha$ is selected properly, then any $(\epsilon, \delta)$-reward free algorithm can be converted into an algorithm that provides the exact estimator with a probability of at least $1 - \delta$.

\begin{lemma}\label{lm:free}
Suppose $\alpha \ge \frac{2\sqrt{2}\epsilon}{H - 1}$, then any $(\epsilon, \delta)$-reward free algorithm could be converted to an algorithm which outputs an estimator $\hat\btheta$, satisfying $\EE_{\btheta \sim \mathrm{Unif}(\bTheta)}[\ind(\hat \btheta = \btheta)] \ge 1 - \delta$.
\end{lemma}

Equipped with these lemmas, we can provide the proof for Theorem~\ref{thm:mainL}.

\begin{proof}[Proof of Theorem~\ref{thm:mainL}]
Set $\alpha = \frac{2\sqrt{2}\epsilon}{H - 1}$ and $\epsilon \le (H - 1) / (2\sqrt{2})$, then by Lemma \ref{lm:free}, any $(\epsilon, \delta)$-reward free algorithm could be converted to an estimation algorithm with successful rate at least $1 - \delta$. Thus from Lemma~\ref{lm:class}, the sample complexity $K$ of these reward free algorithms is bounded by
\begin{align*}
    K &\ge \frac{2 - \alpha^2}{4\alpha^2}\bigg(\frac{d-1}{16} - 3\bigg)(1 - \delta) - \frac{2 - \alpha^2}{4\alpha^2}\log 2 \\
    &\ge \frac{(H - 1)^2}{128\epsilon^2}\bigg(\frac{d-1}{16} - 3\bigg)(1 - \delta) - \frac{(H - 1)^2\log 2}{128\epsilon^2}.
\end{align*}
Suppose $H \ge 2, d \ge 50, \delta \ge 1/2$ to simplify the result, we conclude that there exists an absolute positive constant $C$ such that $K \ge C(1 - \delta)H^2d\epsilon^{-2}$, which leads to our final conclusion.
\end{proof}

\section{Missing Proofs in Appendix~\ref{sec:proofL}}
We provide detailed proofs for lemmas in Appendix~\ref{sec:proofL}. For simplicity, we denote by $d' = d - 1$ the dimension of the binary set $\cM$.

\subsection{Proof of Lemma~\ref{lm:packing}}
\begin{proof}[Proof of Lemma~\ref{lm:packing}]
To begin with, we assume that $\xb \sim \mathrm{Unif}(\cB_{d'})$, i.e. $[\xb]_i \sim \{-1, 1\}$. Thus given any $\xb, \xb' \sim \mathrm{Unif}(\cB_{d'})$, we have
\begin{align*}
    \PP(\la \xb, \xb' \ra \ge d'\gamma) &= \PP_{z_i \sim \mathrm{Unif}\{-1, 1\}}\bigg(\sum_{i=1}^{d'} z_i \ge d'\gamma\bigg) \\
    &= \PP_{z_i \sim \mathrm{Unif}\{-1, 1\}}\bigg(\frac1{d'}\sum_{i=1}^{d'} z_i \ge \gamma\bigg) \\
    &\le \exp(-d'\gamma^2 / 2),
\end{align*}
where the last inequality holds by utilizing the Azuma-Hoeffding's inequality with the fact that $z_i \sim \mathrm{Unif}\{-1, 1\} $ is a bounded random variable. Consider a set $\cM$ with cardinality $|\cM|$, then there is at most $|\cM|^2$ pair of $(\xb, \xb')$. Thus taking a union bound over all vector pairs $(\xb, \xb')$, we have 
\begin{align*}
    \PP(\exists \xb, \xb' \in \cM, \xb \neq \xb', \la \xb, \xb'\ra \ge d'\gamma) \le |\cM|^2\exp(-d'\gamma^2/2),
\end{align*}
thus
\begin{align*}
    \PP(\forall \xb, \xb' \in \cM, \xb \neq \xb', \la \xb, \xb'\ra \le d'\gamma) \ge 1 - |\cM|^2\exp(-d'\gamma^2/2),
\end{align*}
Once we have that $|\cM|^2 \exp(-d'\gamma^2 / 2) < 1$, there exists a set $\cM$ such that for any two different vector $\xb, \xb' \in \cM, \la \xb, \xb'\ra \le d\gamma$.
\end{proof}

\subsection{Proof of Lemma~\ref{lm:class}}
We start our lower bound proof from Fano's inequality.
\begin{lemma}[Fano's inequality, \cite{fano1961transmission}]\label{lm:fano}
Consider probability measures $\PP_{\btheta}, \btheta \in \bTheta$ on space $\Omega$ parameterized by $\btheta \in \bTheta$. Then for any estimator $\hat \btheta$ on $\Omega$ and any comparison law $\PP_0$ on $\Omega$
\begin{align*}
    \frac1{|\bTheta|}\sum_{\btheta \in \bTheta}\PP_{\btheta}[\hat \btheta \neq \btheta] \ge 1 - \frac{\log 2 + \frac1{|\bTheta|}\sum_{\btheta \in \bTheta} \mathrm{KL}(\PP_{\btheta}, \PP_0)}{\log |\bTheta|}.
\end{align*}
\end{lemma}

Then we will start our proof from the deterministic algorithms, which could be further extended to random algorithms using Yao's principle~\citep{yao1977probabilistic}.
\begin{proof}[Proof of Lemma~\ref{lm:class}]
We denote $Y_k$ as such a trajectory at episode $k$ and $Y_{1:k}$ for the trajectories $Y_1, \cdots, Y_k$. We have for the KL divergence over joint distribution $Y_{1:k}$,
\begin{align*}
    &\mathrm{KL}(\PP_{\btheta}(Y_{1:K}), \PP_0(Y_{1:K})) \\
    &\quad= \sum_{Y_{1:K}} \PP_{\btheta}(Y_{1:K}) \log(\PP_{\btheta}(Y_{1:K})/\PP_0(Y_{1:K})) \\
    &\quad= \sum_{Y_{1:K}} \PP_{\btheta}(Y_{1:K - 1})\PP_{\btheta}(Y_K | Y_{1:K - 1}) \log(\PP_{\btheta}(Y_{1:K - 1})/\PP_0(Y_{1:K - 1}))\\
    &\qquad + \sum_{Y_{1:K}} \PP_{\btheta}(Y_{1:K - 1})\PP_{\btheta}(Y_K | Y_{1:K - 1}) \log(\PP_{\btheta}(Y_{1:K - 1})/\PP_0(Y_{1:K - 1}))\\
    &\quad= \sum_{Y_{1:K - 1}} \PP_{\btheta}(Y_{1:K - 1}) \log(\PP_{\btheta}(Y_{1:K - 1})/\PP_0(Y_{1:K - 1}))\sum_{Y_K}\PP_{\btheta}(Y_K | Y_{1:K - 1})\\
    &\qquad + \sum_{Y_{1 : K - 1}} \PP_{\btheta}(Y_{1:K - 1})\sum_{Y_k}\PP_{\btheta}(Y_K | Y_{1:K - 1}) \log(\PP_{\btheta}(Y_{1:K - 1})/\PP_0(Y_{1:K - 1}))\\
    &\quad= \mathrm{KL}(\PP_{\btheta}(Y_{1:K - 1}), \PP_0(Y_{1:K - 1})) + \EE_{Y_{1:K - 1}}[\mathrm{KL}(\PP_{\btheta}(Y_K|Y_{1:K - 1}, \PP_{0}(Y_K|Y_{1:K - 1})].
\end{align*}
Thus by further expanding the above equations, we have
\begin{align*}
    \mathrm{KL}(\PP_{\btheta}(Y_{1:K}), \PP_0(Y_{1:K})) = \sum_{k=1}^K\EE_{Y_{1:k - 1}}[\mathrm{KL}(\PP_{\btheta}(Y_k|Y_{1:k - 1}), \PP_0(Y_k|Y_{1:k - 1})],
\end{align*}
where we denote $\EE_{Y_{1:0}}[\mathrm{KL}(\PP_{\btheta}(Y_1|1:0), \PP_0(Y_1|1:0)] := \mathrm{KL}(\PP_{\btheta}(Y_1), \PP_0(Y_0))$ for consistency.

Since for any deterministic algorithm, in any episode, the trajectory $s_1, a_1, \cdots, s_H, a_H$ is determined after the algorithm goes into $S_{2,1}$ or $S_{2, 2}$, furthermore, for these deterministic algorithms, the first action $a$ at $k$-th trajectory is fixed given previous knowledge $Y_{1:k}$. Therefore, the distribution of the whole trajectory could be replaced by the distribution of $S_{2, 1}$ and $S_{2, 2}$. We have there are at most two possible value for $Y_k$, we denote the trajectory $S_1, a_1, S_{2, 1}, \cdots, S_{2, 1}$ by $Y_k = 0$ and the other trajectory $S_1, a_1, S_{2, 2}, \cdots, S_{2, 2}$ by $Y_k = 1$. We define the comparison distribution $\PP_0$ as 
\begin{align*}
    \PP_{0}(Y_k = 0 | Y_{1:k - 1}) &= \frac1{|\bTheta|}\sum_{\btheta_i \in \bTheta} \PP_{\btheta_i}(Y_k = 0 | Y_{1:k - 1}) := \frac12 + \frac{\alpha}{\sqrt{2}d'}\la \ab, \bar \btheta\ra\\
    \PP_{0}(Y_k = 1 | Y_{1:k - 1}) &= \frac1{|\bTheta|}\sum_{\btheta_i \in \bTheta} \PP_{\btheta_i}(Y_k = 1 | Y_{1:k - 1}) := \frac12 - \frac{\alpha}{\sqrt{2}d'}\la \ab, \bar \btheta\ra,
\end{align*}
where we denote $\bar \btheta$ is the mean value of $\tilde \btheta_i \in \cM$. (Recall that $\bTheta = \big\{\btheta_i | \btheta_i = ( \sqrt{2}, \alpha\tilde \btheta_i^\top / \sqrt{d'})^\top\big\}$). For simplicity, we use $\PP_{\btheta_i}$ and $\PP_0$ to denote the distributions for the whole trajectory defined above. Then we could bound the KL divergence between $\PP_0$ and $\PP_{\btheta_i}$ as 
\begin{align*}
    &\mathrm{KL}(\PP_{\btheta_i}, \PP_0) \\
    &\quad= \bigg(\frac12 + \frac{\alpha}{\sqrt{2}d'}\la \ab, \tilde \btheta_i\ra\bigg)\log\bigg(\frac{\sqrt{2}d' + \alpha\la\ab, \tilde \btheta_i \ra}{\sqrt{2}d' + \alpha\la \ab, \bar \btheta\ra}\bigg) \\
    &\quad+ \bigg(\frac12 - \frac{\alpha}{\sqrt{2}d'}\la \ab, \tilde \btheta_i\ra\bigg)\log\bigg(\frac{\sqrt{2}d' - \alpha\la\ab, \tilde \btheta_i \ra}{\sqrt{2}d' - \alpha\la \ab, \bar \btheta\ra}\bigg)\\
    &\quad\le \bigg(\frac12 + \frac{\alpha}{\sqrt{2}d'}\la \ab, \tilde \btheta_i\ra\bigg) \frac{\alpha\la\ab, \tilde \btheta_i - \bar \btheta\ra}{\sqrt{2}d' + \alpha\la\ab, \bar \btheta\ra} - \bigg(\frac12 - \frac{\alpha}{\sqrt{2}d'}\la \ab, \tilde \btheta_i\ra\bigg) \frac{\alpha\la\ab, \tilde \btheta_i - \bar \btheta\ra}{\sqrt{2}d' - \alpha\la\ab, \bar \btheta\ra}.
\end{align*}
Taking summation over $\btheta_i \in \bTheta$, recall that $\bar \btheta$ is the mean value of $\tilde \btheta_i \in \cM$, we have 
\begin{align*}
    \sum_{\btheta_i \in \bTheta} \mathrm{KL}(\PP_{\btheta_i}, \PP_{0}) = \frac{2\alpha^2}{2d'^2 - \alpha^2\la\ab, \bar \btheta\ra^2}\sum_{\btheta_i \in \bTheta} \la \ab, \tilde \btheta_i\ra\la\ab, \tilde \btheta_i - \bar \btheta\ra.
\end{align*}
Given the fact that $\la\xb, \xb'\ra \le d'$ for any $\xb, \xb' \le \cA$, one can easily get that 
\begin{align*}
    \frac1{|\bTheta|}\sum_{\btheta_i \in \bTheta} \mathrm{KL}(\PP_{\btheta_i}, \PP_{0}) \le \frac{4\alpha^2}{2 - \alpha^2}.
\end{align*}
Plugging the above inequality into the decomposition of KL divergence, from Fano's inequality Lemma~\ref{lm:fano}, we have 
\begin{align*}
    \delta \ge \frac{1}{|\bTheta|} \sum_{\btheta \in \bTheta} \PP_{\btheta}[\hat \btheta \neq \btheta] \ge 1 - \bigg(\frac{d'}{16} - 3\bigg)^{-1}\bigg(\log 2 + \frac{4K\alpha^2}{2 - \alpha^2}\bigg).
\end{align*}
Replacing $d'$ by $d - 1$, we can get the same result as the statement of the lemma.
\end{proof}

\subsection{Proof of Lemma~\ref{lm:free}}
We show the proof for Lemma~\ref{lm:free} by establishing different reward functions for this MDP structure.
\begin{proof}[Proof of Lemma~\ref{lm:free}]
For any $\btheta \in \bTheta$, we build the reward sequence as $r(S_1) = r(S_{2, 1}) = 0, r(S_{2, 2}) = 1$, then any $(\epsilon, \delta)$-correct algorithm guarantees that
\begin{align*}
    \PP(V^*(S_1, r; \btheta) - V^\pi(S_1, r; \btheta) \le \epsilon) \ge 1 - \delta, \forall \btheta \in \bTheta.
\end{align*}
Our proof is to show that, as long as $\alpha \ge \frac{2\sqrt{2}\epsilon}{H - 1}$, we can build up the estimation of $\btheta$ using $\btheta_i = \begin{pmatrix} \sqrt{2}, \alpha \ab_i^\top / \sqrt{d}\end{pmatrix}^\top$ where $\ab_i$ is determined by $a_i = \pi(S_1)$. It is guaranteed that $\PP_{\btheta}[\btheta_i = \btheta] \ge 1 - \delta$ and furthermore, $\EE_{\btheta \sim \mathrm{Unif}(\bTheta)}(\ind[\btheta = \hat \btheta]) \ge 1 - \delta$.

Suppose for the MDP parameter $\btheta_i$, it is easy to find that the optimal policy for the first step is $\pi^*(S_1) = a_i$. Suppose that for any policy $\pi(S_1) = a_j$ where $j \neq i$, then from the MDP structure, the gap between policy and the optimal policy is
\begin{align*}
    V^*(S_1, r; \btheta_i) - V^\pi(S_1, r; \btheta_i) &= \frac{(H - 1)\alpha}{\sqrt{2}d}(\la\ab_i, \ab_i\ra - \la \ab_i, \ab_j\ra) \\
    &\ge \frac{(H - 1)\alpha}{\sqrt{2}d'}(d' - d' / 2) \\
    &= \frac{(H - 1)\alpha}{2\sqrt{2}},
\end{align*}
as long as we have $\alpha \ge \frac{2\sqrt{2}\epsilon}{H - 1}$, we can get the policy gap $V^*(S_1, r; \btheta_i) - V^\pi(S_1, r; \btheta_i) \ge \epsilon$.

Therefore, it is easy to show that the estimation using the policy $\pi(S_1)$ is guaranteed with successful rate at least $1 - \delta$ for any MDP parameter $\btheta_i$, thus we can further conclude that $\EE_{\btheta \sim \mathrm{Unif}(\bTheta)}(\ind[\btheta = \hat \btheta]) \ge 1 - \delta$.
\end{proof}

\end{document}